\documentclass[a4paper,fleqn]{cas-dc}

\usepackage{amsmath,amssymb,amsfonts}
\usepackage{mathtools}
\usepackage{bm}
\usepackage{amsthm}
\allowdisplaybreaks
\usepackage{amsmath,amssymb,amsthm,mathtools,graphicx}
\allowdisplaybreaks
\newtheorem{lemma}{Lemma}
\newtheorem{proposition}{Proposition}
\usepackage{graphicx}
\usepackage{booktabs}
\usepackage{multirow}
\usepackage{soul}
\usepackage{array}
\usepackage{xcolor,colortbl}
\usepackage{changepage,threeparttable}
\usepackage[caption=false,font=footnotesize]{subfig}
\usepackage{amsmath, amssymb}
\usepackage{stfloats}   % better float placement for two-column
\usepackage{changepage} % in the preamble
\usepackage{float}
\usepackage{cite}
\usepackage{cleveref}
\usepackage{makecell,adjustbox}
\usepackage[most]{tcolorbox}      % định nghĩa môi trường tcolorbox

\tcbuselibrary{breakable,skins}   % để dùng key 'breakable' và các skin
\usepackage{amsmath,amssymb,amsthm,mathtools}

\usepackage[protrusion=true,expansion=true]{microtype}
\usepackage{enumitem} % compact lists
\setlist{nosep,leftmargin=*,itemsep=0pt,topsep=2pt,parsep=0pt}

\usepackage{IEEEtrantools} % for IEEEeqnarray (multi-line eqns in 2 columns)
\usepackage{mathtools}     % for \mathclap, \smash, etc.
\usepackage[protrusion=true,expansion=true]{microtype}
\usepackage{scrextend}
\makeatletter
\newcommand{\removelatexerror}{\let\@latex@error\@gobble}

\newcommand{\tool}{\textsc{BRIDGE}\xspace}

\usepackage{placeins}

\let\Oldsection\section
\renewcommand{\section}{\FloatBarrier\Oldsection}

\makeatother
\theoremstyle{plain}
\newtheorem{theorem}{Theorem}
\theoremstyle{definition}
\newtheorem{definition}{Definition}
\newtheorem{corollary}{Corollary}
\theoremstyle{remark}
\newtheorem{assumption}{Assumption}
\newtheorem{remark}{Remark}
\setlist{nosep,leftmargin=1.25em,labelindent=0pt}

\begin{document}

\def\floatpagepagefraction{1}
\def\textpagefraction{.001}

\shorttitle{BRIDGE}
% Set short authors for running header/footer
\shortauthors{Le \textit{et al.}}
\title [mode = title]{BRIDGE: Budget-aware Reasoning via Intermediate Distillation with Guided Examples}

\author{Xuan-An Le}
% [orcid=xxx]
\ead{22028131@vnu.edu.vn}

\affiliation{organization={Faculty of Information Technology, VNU University of Engineering and Technology},
    city={Hanoi},
    country={Vietnam}}

\author{Minh-Nam Tran}
% [orcid=0009-0007-7318-6896]
\ead{23021646@vnu.edu.vn}

% \author{Thu-Trang Nguyen}
% [orcid=0000-0002-3596-2352]
% \ead{trang.nguyen@vnu.edu.vn}

\author{Son Nguyen}
% [orcid=0000-0002-8970-9870]
\ead{sonnguyen@vnu.edu.vn}
\cormark[1]

% \author{Hieu Dinh Vo}
% [orcid=0000-0002-9407-1971]
% \ead{hieuvd@vnu.edu.vn}

% Corresponding author text
\cortext[cor1]{Corresponding author}

\begin{abstract}
Distilling knowledge from large proprietary models (e.g., GPT-4) to tiny deployable models ($\le$1B parameters) faces a critical \textit{capacity-budget trap}: the $1000\times$ capacity gap between teachers and students prevents effective direct transfer, while API costs prohibit extensive data collection.
We introduce \tool (Budget-Aware Reasoning via Intermediate Distillation), a two-phase framework that resolves these constraints through strategic intermediation and budget asymmetry.
In Phase~1, a mid-sized Teacher Assistant (TA; e.g., $\sim$7B) learns from the black-box teacher on a strictly limited subset of data (e.g., 3-5\%), selected via a \textit{zero-API-cost pipeline} that balances entropic difficulty and semantic diversity using only local TA inference.
In Phase~2, we exploit this asymmetry-teacher queries are expensive, whereas TA inference is free to amplify supervision: the refined TA generates synthetic rationales for the full dataset to train the tiny student.
Crucially, we apply an instruction-tuning curriculum to establish behavioral alignment in the tiny student before transferring reasoning.
Our theoretical analysis shows that \tool yields tighter generalization bounds than direct distillation when data is abundant.
Experiments across medical, legal, and financial benchmarks demonstrate consistent improvements: \tool delivers student performance gains of 28-41\%, closing the capability gap with proprietary teachers by 12-16\% while using $10\times$ fewer teacher queries.
Notably, \tool defies the conventional cost-performance frontier, surpassing direct distillation baselines that use 100\% of the budget while consuming only 5\% of the resources.
\end{abstract}


\maketitle

\section{Introduction}

Large language models (LLMs), such as GPT-4~\cite{openai_gpt4_docs}, Gemini~\cite{gemini2024}, Grok~\cite{xai2023grok}, and DeepSeek~\cite{deepseek2024}, achieve remarkable performance in reasoning and general-purpose capabilities~\cite{wei2022chain,kojima2022large}. However, their deployment in real-world scenarios is increasingly constrained by privacy requirements, latency sensitivity, and prohibitive API costs~\cite{zhang2024tinyllama,xia2024sheared}. 
This has accelerated a paradigm shift toward ``\textit{Tiny Models}'' ($\le$1B parameters) designed for efficient on-device inference~\cite{zhang2024tinyllama,xia2024sheared}. 
Yet, effectively transferring the complex capabilities of trillion-parameter teachers to such compact students remains a formidable challenge.

This transfer faces two considerable challenges, which we term the \textit{capacity-budget trap}. 
First, the capacity gap between proprietary teachers ($>$1T parameters) and tiny students ($<$1B) exceeds 1000$\times$ far beyond where direct distillation succeeds. 
At this scale, direct distillation notoriously fails because the student lacks the representational capacity to approximate the teacher's complex decision boundaries~\cite{mirzadeh2019improvedknowledgedistillationteacher, cho2019efficacy}.
Second, the \textit{budget constraint} arises because modern experts are black-box APIs with significant per-token costs. While overcoming the capacity gap typically demands massive supervision~\cite{gu2024minillm,agarwal2024gkd}, strict budgets ($B$) force a compromise: one must choose between insufficient supervision (low cost) or prohibitive expense (high coverage).

% Existing approaches and their limitations
Existing methods attempt to optimize \textit{direct} teacher-to-student transfer. Black-box approaches include instruction tuning~\cite{ouyang2022instructgpt}, rationale distillation~\cite{hsieh-etal-2023-distilling,wei2022chain}, and adversarial generation~\cite{lion}. 
These approaches require massive supervision.
Meanwhile, white-box methods like ULD~\cite{ULD} align cross-tokenizer vocabularies. Yet all share a fundamental limitation: they address signal quality rather than structural incompatibility. Algorithmic refinements cannot enable a sub-billion parameter model to directly internalize the latent representations of a trillion-parameter expert under conditions of data poverty.

To solve this, we leverage a critical yet overlooked budget asymmetry: while teacher queries are expensive, inference on locally-hosted intermediate models is free. 
Building on the principle of \textit{scaffolding}, we introduce \tool (Budget-Aware Reasoning via Intermediate Distillation), a two-phase framework that decomposes the distillation problem into manageable stages.

In \textit{Phase 1 (Apprenticeship)}, a mid-sized Teacher Assistant (TA, e.g., $\sim$7B) bridges the capacity gap by learning from the Teacher on a strategically selected subset of data. We employ a selection pipeline that incurs zero API cost-relying solely on local TA forward passes and clustering-to balance difficulty (entropy) and diversity while maximizing information gain per budget unit. 
In \textit{Phase~2 (Tutoring)}, we exploit the budget asymmetry: the refined TA labels the \emph{full} dataset at zero marginal cost to supervise the tiny student. Crucially, we enforce an \textit{instruction-first curriculum}, ensuring the tiny student achieves behavioral alignment before attempting to parse complex reasoning traces.

We evaluate \tool across three specialized domains with widely-used benchmarks: medical (MedConceptsQA)~\cite{shoham2024medconceptsqa}, legal (MMLU-Law)~\cite{hendrycks2020mmlu}, and financial (FOMC)~\cite{tang2025finmteb} distilling popular powerful LLMs including GPT-4, Gemini-2.5 and DeepSeek-V3, into sub-1B (including BLOOMZ-560M, OPT-350M, and Pythia-410M) students via 7B-13B TAs (including Qwen2.5-7B-Instr. and Mistral-7B-Instr.-v0.3 ). 
Results show consistent gains: \tool achieves student gains of 23-41\% for students using \textbf{10$\times$ fewer queries}, and effectively closing the capability gaps with mid-sized TAs by 22-31\% and with proprietary teachers by 10-16\%. 
Crucially, \tool defies the conventional cost-performance pareto frontier: even at a meager 3-5\% budget, it surpasses direct distillation baselines utilizing 100\% of the resources, establishing a new standard for data-efficient knowledge transfer on the edge.

% Contributions
In brief, this paper makes the following contributions:
\begin{enumerate}
    \item \textbf{Methodological:} A two-phase framework decomposing capacity gaps via intermediate TAs, with budget-aware selection and instruction-first training.
    
    \item \textbf{Theoretical:} Analysis establishing when two-stage distillation achieves tighter bounds than direct transfer.
    
    \item \textbf{Empirical:} Extensive experiments showing an improvement of 23-41\% with 10$\times$ fewer queries, validated through ablations on components, scaling, and budget sensitivity.
\end{enumerate}

\section{Preliminaries}
This section formalizes the knowledge distillation paradigms relevant to our work, define the capacity gap phenomenon, and present a budget-constrained problem formulation.

\subsection{Knowledge Distillation Framework}
Knowledge distillation (KD)~\cite{hinton2015distillingknowledgeneuralnetwork} aims to transfer capabilities from a large, high-capacity teacher model $M_T$ to a compact student model $M_S$.
The methodology depends fundamentally on the level of access to the teacher.

\textbf{White-Box Distillation.}
In the traditional setting~\cite{hinton2015distillingknowledgeneuralnetwork}, the student has full access to the teacher's logits and is trained to minimize KL divergence to softened teacher distributions. This approach relies on dense signals from the teacher's full probability space--access that proprietary LLMs do not provide.

\textbf{Black-Box Distillation (Token-Based).}
Modern large language models (LLMs) like GPT-4 or Gemini often serve as black-box teachers, exposing only final text outputs via an API, without access to logits or internal states. This restriction forces a shift from distribution matching to sequence-level supervision.
\begin{itemize}
    \item \textit{Standard Instruction Tuning:} The teacher acts as an oracle to generate targets $y_T$ given an input $x$. The student minimizes the negative log-likelihood (NLL) of the token sequence:
    $$
    \mathcal{L}_{\mathrm{SFT}} = -\sum_{t=1}^{|y_T|} \log P_{M_S}(y_{t} \mid x, y_{<t})
    $$
    While effective for surface-level alignment, this approach discards the teacher's latent reasoning, often leading to superficial pattern matching~\cite{ding2023enhancingchatlanguagemodels}.

    \item \textit{Rationale Distillation:} To capture the teacher's reasoning process, this paradigm elicits a Chain-of-Thought (CoT) rationale $r$ alongside the final answer $a$~\cite{wei2022chain,hsieh-etal-2023-distilling}. The student is supervised on the joint sequence $(r, a)$, maximizing:
    $$
    \label{eq:rationale_loss}
    \mathcal{L}_{\mathrm{CoT}} = \mathcal{L}_{\mathrm{rationale}}(r \mid x) + \lambda \mathcal{L}_{\mathrm{answer}}(a \mid x, r).
    $$
    By explicitly modeling the reasoning path $r$, the student learns \textit{how} to solve the problem rather than merely memorizing the answer $a$.
\end{itemize}

\subsection{Problem Formulation}
We formalize the problem of \textit{Budget-Aware Black-Box Distillation} as follows.

\noindent\textbf{Components:} The setting is defined by three primary entities:

\begin{itemize}
\item \textbf{\textit{Dataset}:} A labeled dataset $\mathcal{D} = \{(x_i, y_i)\}_{i=1}^n$ containing $n$ input-label pairs.

\item \textbf{\textit{Black-Box Teacher} ($M_T$):} A proprietary expert LLM accessible only via API. For an input $x$, it returns a text sequence (rationale and answer) at a monetary cost.

\item \textbf{\textit{Tiny Student} ($M_S$):} A highly compact model (e.g., $\le$1B parameters) with parameters $\theta_S$ intended for deployment.
\end{itemize}

\noindent\textbf{Constraints:} We define a global query budget $B \ll n$, representing the maximum number of API calls allowed to $M_T$. 
We adopt query-level rather than token-level budgets for three reasons: (1) \emph{Practical billing alignment}: commercial APIs (e.g., OpenAI, Anthropic) increasingly use per-request pricing tiers or rate limits alongside token costs, making query counts a natural constraint; (2) \emph{Selection tractability}: query-level budgets enable discrete subset selection ($|S_B| = B$), which admits well-studied combinatorial optimization (e.g., coverage maximization), whereas token budgets would require solving a knapsack-like problem with variable-length outputs; (3) \emph{Reproducibility}: query counts are deterministic and model-agnostic, while token counts depend on tokenizer choice and response length variability.
Let $q(x)$ be an indicator function where $q(x)=1$ if $M_T$ is queried for $x$, and $0$ otherwise. The budget constraint is:
$
\sum_{i=1}^n q(x_i) \le B.
$
Crucially, this budget applies only to $M_T$. We assume access to local computational resources where inference or training on open-weights models is effectively cost-free relative to the API budget.

\noindent\textbf{Objective:}
Our goal is to optimize the parameters $\theta_S$ of the student model to maximize performance on a validation metric $\mathcal{M}$: 
$$
\theta_S^* = \operatorname*{argmax}{\theta_S} \mathbb{E}{(x, y) \sim \mathcal{D}_{\text{test}}} [\mathcal{M}(M_S(x; \theta_S), y)] 
$$
such that $\text{Queries}(M_T) \le B$.
The core difficulty lies in the \textit{capacity-budget trade-off}. Direct distillation requires massive supervision ($B \approx n$) to overcome the capacity gap between $M_T$ and $M_S$, but the budget constraint ($B \ll n$) prohibits this. A valid solution must therefore amplify the sparse supervision signal $B$ into a dense training signal for $M_S$ without exceeding the API limit.

% We resolve this by having the TA efficiently learn from $T$ on a selected $B$-sized subset, then transferring knowledge to $S$ using all $n$ examples without additional teacher queries.

\section{Budget-Aware Reasoning via Intermediate Distillation with Guided Examples}
\newcommand{\blend}{\alpha}            % mixture weight for difficulty/diversity
\newcommand{\ratweight}{\beta}         % rationale / answer weighting in loss
% Removed unused \regw macro to avoid confusion

\subsection{Approach Overview}

\begin{figure*}[t]
    \centering
    \includegraphics[width=1\linewidth]{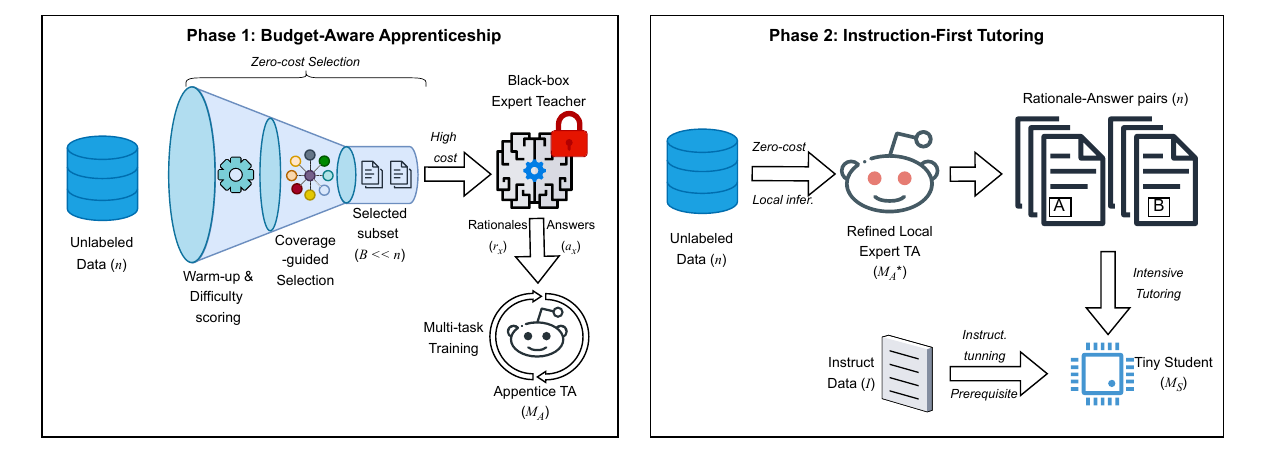}
    \caption{\tool: \textbf{Approach Overview}}
    \label{fig:approach-overview}
\end{figure*}

Figure~\ref{fig:approach-overview} illustrates the \tool architecture, designed to resolve the \textit{capacity-budget trap} inherent in distilling knowledge from trillion-parameter experts ($M_T$) to sub-billion parameter students ($M_S$).
Our main idea is \textit{Budget Asymmetry}, while proprietary teacher queries are expensive (limited to $B$), inference on local open-weights models is effectively free. We exploit this asymmetry by introducing an intermediate Teacher Assistant ($M$) and decomposing the distillation pipeline into two distinct phases:

\textbf{Phase 1 (Budget-aware Apprenticeship).} A mid-sized TA (e.g., $\sim$7B parameters) acts as an \textit{active learner}, identifying and learning from the $B$ most informative examples provided by the black-box Teacher. Through coverage-guided sampling that balances \textit{predictive difficulty} and \textit{semantic diversity}, the TA learns not only final answers but also complete reasoning traces, acquiring knowledge that goes far beyond surface pattern matching.

\textbf{Phase 2 (Instruction-First Tutoring).} The TA acts as a \textit{locally hosted expert}. It labels the entire dataset ($n$) at no marginal cost to train the student extensively. In particular, we enforce an instruction-first curriculum to ensure the tiny student ($\le$1B parameters) can interpret reasoning traces.

\subsection{Phase 1: Budget-Aware TA Apprenticeship}

The objective of Phase 1 is to train TA model $M$ under a limited budget of teacher queries ($B \ll n$).
We formulate this process as a constrained active learning problem in which $M$ must identify a subset $S_B \subset \mathcal{D}$ such that querying $M_T$ on $S_B$ maximizes $M$'s generalization capability.

Inspired by human apprenticeship, where trainees learn most from observing experts on \emph{challenging and diverse} problems, we design a four-stage selection pipeline that balances difficulty (where the TA struggles or is uncertain) and diversity (broad semantic coverage). The proposed four-stage curriculum exploits a critical cost asymmetry: while querying the Teacher $T$ is expensive, inference with $M$ is free. In this pipeline, $M$ stabilizes, self-diagnoses, and optimizes the selection \emph{before} consuming any budget:

\begin{enumerate}
\item \textbf{\textit{Warm-up}} stage: Stabilize $M$ via self-supervised learning on a disjoint subset to establish a baseline for uncertainty estimation;

\item \textbf{\textit{Difficulty Scoring}} stage: Use the warmed-up $M$ to quantify example difficulty using predictive confidence;

\item \textbf{\textit{Coverage-guided Selection}} stage: Cluster the embedding space $z(x) = \mathrm{repr}_{M}(x)$ to enforce semantic coverage, then rank examples by a joint difficulty-diversity score;

\item \textbf{\textit{Rationale Distillation}} stage: Query $T$ exactly once on the optimized subset $S_B$ to obtain rationale-answer pairs $(r_x, a_x)$, and train $M$ using a multi-task objective.
\end{enumerate}

\subsubsection{Warm-up Stage}

% Before the TA can reliably estimate which examples are difficult, it needs a basic understanding of the data distribution. Asking a randomly initialized TA to score difficulty would yield meaningless estimates.

Reliable difficulty scoring requires a calibrated model. Indeed, a randomly initialized TA yields degenerate uncertainty estimates, rendering active selection ineffective.
To mitigate this cold-start problem, we sample a small subset $\mathcal{D}_{\text{warmup}} \subset \mathcal{D}$ (typically 10\% of data) and train $M$ using only self-supervised or label-only objectives that incurs zero teacher cost. For QA tasks, this could be next-token prediction on concatenated (input, label) sequences or minimizing NLL on ground-truth answers $y^*(x)$.
Crucially, to maintain strict budget discipline, we \emph{permanently exclude} $\mathcal{D}_{\text{warmup}}$ from teacher queries. All subsequent selection operates on the reduced pool
$\mathcal{D}' \;=\; \mathcal{D} \setminus \mathcal{D}_{\text{warmup}}$
ensuring that the query budget $B$ is reserved exclusively for novel, unseen examples. 
%
% This disciplined budget accounting is essential for fair comparison with baselines that may use teacher queries less strategically.

\subsubsection{Difficulty Scoring Stage}

% The TA should prioritize learning from examples where it currently performs poorly-these reveal knowledge gaps that teacher supervision can address. Querying the teacher on examples where the TA is already confident would waste budget.

To align the curriculum with TA's specific shortcomings, flaws, or limitations, we prioritize examples where the warmed-up model $M$ exhibits high prediction error.
For each $(x,y^*(x)) \in \mathcal{D}'$, we compute TA's predictive distribution $p_\theta(y\mid x)$ and define a loss-based difficulty score:
$$
\mathrm{Diff}(x)
  = - \log p_\theta\big(y^*(x)\mid x\big).
$$
High difficulty indicates that $M$ is either uncertain (high entropy) or confidently incorrect. Both cases benefit from teacher supervision.

To ensure numerical stability and comparability when combining with diversity metrics, we normalize these raw scores to $[0,1]$ over $\mathcal{D}'$:
$$
\bar{D}(x)
  = \frac{\mathrm{Diff}(x)}
         {\max_{z\in\mathcal{D}'} \mathrm{Diff}(z) + \varepsilon},
$$
where $\varepsilon>0$ ensures numerical stability. 
This scoring process relies exclusively on local forward passes of $M$, allowing us to evaluate the informational utility of every example in $\mathcal{D}'$ without consuming any of the teacher query budget.

\paragraph{Labeled data requirement.}
This score requires ground-truth labels $y^*(x)$, restricting \tool to labeled settings. We justify this design choice in Section~\ref{sec:experiments} (Threats to Validity): our target domains are inherently labeled, and loss-based difficulty provides a stronger signal than entropy alone.

\subsubsection{Coverage-guided Selection Stage}

Relying solely on difficulty scores entails a risk of semantic collapse. Indeed, selecting only the hardest examples risks disproportionately concentrating queries in a narrow and high-error region (e.g., a specific medical subdomain), leaving other areas under-supervised. 
%
%
% The TA trained on such a skewed subset would generalize poorly in Phase~2. 

To mitigate this and ensure the TA generalizes well in Phase 2, we enforce the data selection constraints of both \textit{difficulty} (to maximize learning) and \textit{diversity} (to ensure coverage).
%
% We combine difficulty with embedding-based diversity and enforce coverage via region-wise budget allocation. 
%
This balances exploitation (query where TA struggles) with exploration (maintain broad representation).

\textbf{Semantic Clustering.}
%
% Extract TA embeddings $z(x) = \mathrm{repr}_{M}(x)$ for all $x\in\mathcal{D}'$ and cluster into $C$ regions $\{R_1,\dots,R_C\}$ using GMM or $k$-means. 
%
First, we partition the dataset $\mathcal{D}'$ into $C$ distinct semantic regions $\{R_1, \dots, R_C\}$ via clustering (e.g., $K$-means or \textit{GMM}) on the fixed TA embeddings $z(x) = \mathrm{repr}_{M}(x)$.
%
% Each region captures a coarse semantic grouping (e.g., cardiology vs.\ neurology in medical QA, or contract law vs.\ tort law in legal reasoning).
%
%
To guarantee equitable coverage, we allocate the total budget $B$ across these regions proportional to their probability mass. We assign each region a quota $q_c$:
%
% For each region $R_c$, we estimate its mass $w_c = |R_c| / |\mathcal{D}'|$ and assign it a \emph{fixed quota}

$$q_c = \max\Big\{ q_{\min},\; \big\lfloor w_c B \big\rfloor \Big\}$$
where $w_c = |R_c| / |\mathcal{D}'|$ represents the region's mass/density, and $q_{\min}$ acts as a protective floor to prevent mode collapse in sparse regions, $q_{\min} = \max\{1, \lfloor B/(10C)\rfloor\}$.
% enforces that even small regions receive at least a few teacher calls.
%
We then renormalize the $\{q_c\}$ if necessary so that $\sum_{c=1}^C q_c = B$. This defines a simple, interpretable budget allocation such that large regions get more queries, but no region is starved.

\textbf{Difficulty-Diversity Scoring and Ranking.}
Within each region $R_c$, we rank examples that are both difficult and distinctive.
This can be done via two passes:

\begin{itemize}

\item \textbf{\textit{Diversity seeding}:} Within each region $R_c$, we select a small number $m_0$ of \emph{seed} examples using Farthest-Point Sampling~\cite{sener2018coreset}. This yields an initial seed set $S_0$ that is well-spread over the representation manifold.

\item \textbf{\textit{Score Calculation}:} For each $x\in\mathcal{D}'$, we calculate a normalized diversity score $\bar{n}(x)$ based on its Euclidean distance to the nearest seed:
$$
\bar{n}(x)
  = \frac{\min_{z\in S_0} d_\phi(x,z) - n_{\min}}
         {n_{\max} - n_{\min} + \epsilon},
$$
where $n_{\min}$ and $n_{\max}$ are the minimum and maximum nearest-seed distances over $\mathcal{D}'$, $\epsilon>0$ is a small value to avoid division by zero.
%
% Intuitively, $\bar{n}(x)$ is high when $x$ lies far from existing seeds.
%
We then synthesize the final selection metric $s(x)$ as a convex combination of the difficulty score (from \textit{Difficulty Scoring} stage) and the normalized diversity score:
$$
s(x)
  = \blend\,\bar{D}(x) + (1-\blend)\,\bar{n}(x),
$$
where $\alpha \in [0,1]$ modulates the trade-off. 
\end{itemize}
Intuitively, this function rewards examples that TA finds hard ($\bar{D}(x)$) but which are also semantically distinct from the examples already implicitly covered by the seeds ($\bar{n}(x)$).

\textbf{Region-wise top-$q_c$ Selection.}
Finally, in each region $R_c$ we sort examples by $s(x)$ in descending
order and select the top-$q_c$ examples:
$$
S_c = \operatorname*{arg\,top}_{x\in R_c}^{{}\;q_c} s(x).
$$
The overall teacher-queried subset is then
$$
S_B = \bigcup_{c=1}^C S_c,
\qquad |S_B| = B.
$$

% \textbf{Why this design works:}

% \begin{itemize}[leftmargin=*,itemsep=2pt,topsep=2pt]
% \item \textbf{Zero teacher queries for selection}: All scoring (difficulty, diversity, clustering) uses only TA-side computation, preserving budget for actual distillation;
% \item \textbf{Budget transparency}: Per-region quotas $\{q_c\}$ sum exactly to $B$, ensuring no budget over/underspend;
% \item \textbf{Interpretable trade-offs}: Parameter $\blend$ controls difficulty-diversity balance; $C$ and $q_{\min}$ control coverage granularity-all independent knobs that practitioners can tune;
% \item \textbf{Empirical validation}: Ablations (Table~\ref{tab:perf_sel_learning_ablation}) show selective learning boosts TAs by +10.12 pts and students by +5.5 pts vs.\ random selection, confirming the strategy's effectiveness.
% \end{itemize}

% \paragraph{Practical complexity.}
% In practice, we implement Stage~2 with low-dimensional TA embeddings (256-768
% dimensions) and a modest number of regions $C\in\{16,32\}$, so that clustering
% and farthest-point sampling are inexpensive even for $n\approx 10^5$. We run
% farthest-point sampling only once per region to obtain seed set $S_0$, and then
% compute nearest-seed distances in a single batched pass, yielding $O(nd)$
% complexity with small constants. In our experiments, the entire selection phase
% accounts for $<5\%$ of the end-to-end training time, while simpler alternatives
% (e.g., random seeding) lead to noticeably worse student performance
% (Section~\ref{sec:experiments}).

\subsubsection{Multi-Task Rationale Distillation Stage}

With the strategically optimized subset $S_B$ established, we query the black-box Teacher \emph{exactly once} on each $x\in S_B$, requesting both a Chain-of-Thought (CoT) rationale $r_x$ and final answer $a_x$. This step consumes $B$ units of the query budget, finalizing the data acquisition phase.

To maximize the utility of these expensive annotations, we treat the rationale as a necessary scaffold for the answer. We train TA to maximize a multi-task objective that models the joint probability of the reasoning trace and the conclusion:
$$
\mathcal{L}= \mathrm{NLL}(r_x\mid x)+ \beta \mathrm{NLL}(a_x\mid x, r_x)
$$
Here the first term forces TA to internalize the teacher's deductive process, while the second term conditions the final answer $a_x$ on both the input $x$ and the generated rationale $r_x$. The hyperparameter  $\beta>0$ balances the two tasks~\cite{ho2022large,shridhar2023distilling}.

Recent work~\cite{hsieh2023distillingstepbystepoutperforminglarger} shows rationale-augmented distillation enables smaller models to internalize \emph{reasoning processes}, not just input-output mappings. This yields better generalization and interpretability. Our experiments (Section~6.2.1) confirm multi-task training is critical.

% \textbf{Training.} We fine-tune for $E$ epochs on $S_B$ with batch size $b$ and learning rate schedule $\eta(e)$. The resulting $M^\star$ serves as the teacher assistant for Phase~2. Algorithm~\ref{alg:phase1-v2} summarizes the complete apprenticeship pipeline.

This phase yields the fully refined Teacher Assistant, denoted $M^\star$. This model, having learned from the most informative examples in the domain, is now capable of acting as a high-fidelity surrogate for the teacher in the subsequent intensively tutoring phase.

\subsection{Phase 2: Instruction-First Tutoring}

The goal of Phase 2 is to distill the reasoning capabilities of the refined Teacher Assistant ($M^\star$) into the deployment-ready Student ($M_S$).
By deploying $M^\star$ to label the \textbf{entire} dataset $\mathcal{D}$, we increase the training signal from $B$ examples to $n$ examples, providing the student with dense, process-oriented supervision across the full data distribution.

%
% \textbf{This asymmetry is \tool's key efficiency advantage}: direct distillation must choose between limited teacher queries ($B$ examples) or prohibitive costs ($n$ queries); \tool gets both-selective teacher supervision \textbf{plus} complete TA supervision.

To ensure the low-capacity student effectively internalizes this dense supervision, we implement Phase 2 in three stages:
% This sequencing is critical to prevent the "learning collapse" often observed when tiny models attempt to mimic complex Chain-of-Thought traces without prior alignment.

\begin{itemize}
    \item \textbf{Instruction Alignment}: Untreated tiny models ($\le$1B) often lack the baseline capability to parse complex reasoning prompts, resulting in repetition or topic drift. 
    Before distillation, we fine-tune $M_S$ on a curated instruction dataset $\mathcal{I}$ (e.g., \texttt{databricks-dolly-15k}~\cite{DatabsricksDolly}, a 15K-example dataset of diverse instruction-response pairs). This step establishes behavioral compliance, transforming $M_S$ into a reliable instruction-following agent capable of interpreting TA's outputs.
    
    \item \textbf{Synthetic Annotation}: 
    % Query the refined assistant $M_A$ once on each input to cache rationale-answer pairs for all $x\in\mathcal{D}$;
    %
    We perform inference on the full dataset $\mathcal{D}$ with $M^\star$. For every input $x$, we generate and cache a rationale-answer pair $(r_x, a_x)$. This effectively transfers the TA's internalized knowledge acquired from Teacher in Phase 1.
    
    \item \textbf{Rationale Distillation}: 
    % Train the student on cached rationale+answer pairs from the assistant using a multi-task objective.
    %
    We train the aligned $M_S$ on the cached dataset using the multi-task objective. By minimizing the joint loss for rationales and answers, $M_S$ learns to imitate TA's deductive process rather than merely memorizing final labels, achieving robust generalization despite its limited capacity.
\end{itemize}

\subsubsection{Instruction Alignment}

To bridge the capability gap between the tiny student and complex reasoning tasks, we first establish behavioral compliance. Particularly, we fine-tune student $M_S$ on instruction dataset $\mathcal{I}$ for $E_{\text{warmup}}$ epochs that minimizes the standard supervised loss:
$$
\mathcal{L}_{\mathrm{SFT}} = \mathrm{NLL}(y \mid x), \qquad (x,y)\in\mathcal{I}.
$$
This instruction tuning transforms the student into an instruction-following model that can produce coherent, on-topic responses. 
%
%
% The examples below (from \texttt{bloomz-560m}) were \textbf{not} in the training set, demonstrating generalization:
%
As illustrated in the examples below in Figure~\ref{fig:example-instr} (using BLOOMZ-560M), this alignment is a strict prerequisite: without it, the student fails to parse CoT format, rendering downstream distillation ineffective.

\begin{figure}
    \centering
    \includegraphics[width=1\linewidth]{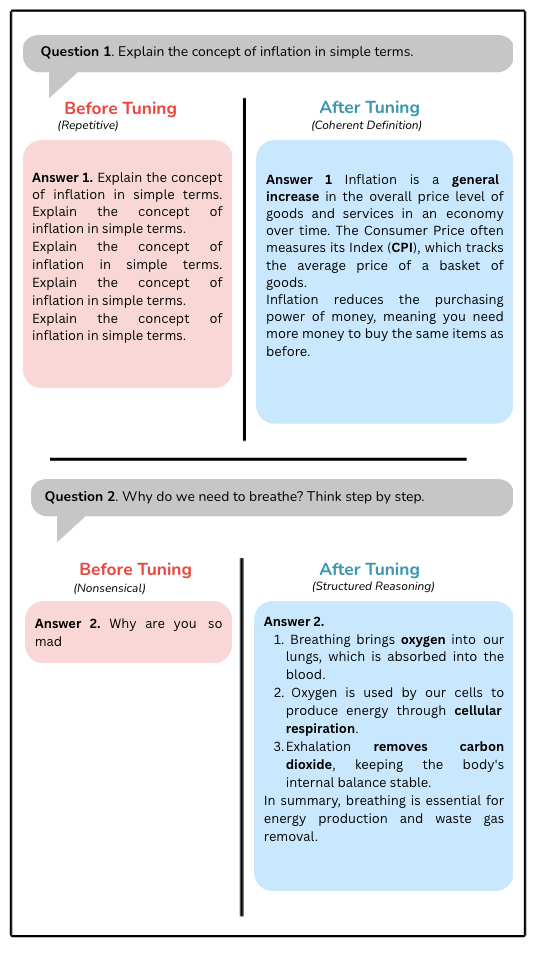}
    \caption{Example QA for Model BLOOMZ-560M before and after
   Instruction Alignment}
    \label{fig:example-instr}
\end{figure}

\subsubsection{Synthetic Annotation}

As $M^\star$ runs locally, the \textbf{entire} dataset $\mathcal{D}$ (including $\mathcal{D}_{\text{warmup}}$ excluded from teacher queries) is used to train the student. 
This constitutes \tool's key efficiency advantage, allowing for comprehensive supervision ($n$ examples) while paying for only selective supervision ($B$ teacher queries).

To avoid redundant TA inference during student training epochs, we query $M^\star$ \textbf{once} per example upfront:
$$
(r_x,a_x) \gets M^\star(x), \qquad \forall x \in \mathcal{D},
$$
storing triplets $(x, r_x, a_x)$ in cache $\mathcal{C}$. This one-time cost (single TA forward pass per example) amortizes across all $E$ training epochs.

\subsubsection{Rationale Distillation}

We train the aligned student $M_S$ on the cached dataset $\mathcal{C}$ using a multi-task, rationale-augmented objective.
% .that transfers both reasoning processes and final predictions. 
This forces the student to mimic TA's deductive process rather than merely memorizing the final outputs.
For each cached triplet $(x, r_x, a_x) \in \mathcal{C}$, the loss function is defined as:
$$
\mathcal{L} = \mathrm{NLL}(r_x\mid x) + \beta\mathrm{NLL}(a_x\mid x, r_x),
$$
where $\beta \in (0,1]$ balances the reasoning and prediction tasks. The student generates $r_x$ first, then conditions answer $a_x$ on both input and generated rationale-mirroring the TA's (and teacher's) Chain-of-Thought process.
%
% Recent distillation studies~\cite{hsieh2023distillingstepbystepoutperforminglarger,magister2023teaching} show rationale-augmented supervision enables smaller models to internalize \textbf{how to reason}, not just input-output mappings. 
%
% By exposing students to the TA's step-by-step explanations across all $n$ examples, we achieve deeper knowledge transfer than label-only training.
%
% Our experimental results (Section~\ref{}) confirm multi-task objectives are essential for strong performance.

% Optimize for $E$ epochs on $\mathcal{C}$ with batch size $b$ and learning rate schedule $\eta(e)$. Critically, this entire stage uses \textbf{zero} additional teacher queries-all supervision comes from the cached TA outputs. Algorithm~\ref{alg:phase2} summarizes Phase~2.

% \textbf{Final remarks on \tool's design.} By decoupling expensive teacher queries (Phase~1: $B$ examples, strategic selection) from comprehensive knowledge transfer (Phase~2: $n$ examples, free TA supervision), \tool achieves order-of-magnitude efficiency gains. Section~6 validates this empirically: \tool at 3-5\% teacher budget outperforms direct distillation at 100\% budget, delivering 4-11 point student improvements and 12-28 point TA improvements across three challenging domains.

\paragraph{Theoretical Abstraction.}
For the theoretical development in Section~\ref{sec:theory} we adopt a slightly more abstract notation. We denote the black-box teacher by $M_T\in\mathcal{F}_T$, the refined teacher assistant  by $M_A\in\mathcal{F}_A$, and the final student by $M_S\in\mathcal{F}_S$. 
$M_T$ corresponds to the (black-box) Teacher, $M_A$ to the fully refined TA $M^\star$ obtained at the end of Phase~1, and $M_S$ to the deployment-ready Student. The analysis focuses on the simplified pathway $M_T \to M_A \to M_S$ and treats auxiliary warm-up and instruction-tuning steps as improving initialization and reducing optimization error; omitting them makes our bounds slightly pessimistic but does not affect the qualitative comparison between \textsc{\tool} and
direct distillation. 

For clarity, the teacher's query budget $B$ is measured at the example level: querying $M_T$ on a single input $x$ returns a rationale-answer pair and consumes one unit of budget. In the experiments (Section~\ref{sec:experiments}), we additionally track token-level costs and verify that our conclusions are robust under more precise API pricing models (e.g., rationales vs.\ answers only).

\section{Theoretical Analysis}
\label{sec:theory}

We provide a risk-decomposition analysis that characterizes \emph{when} the two-stage BRIDGE pipeline $M_T \to M_A \to M_S$ can outperform direct distillation. Rather than claiming tight bounds with precisely quantified constants which would require strong distributional assumptions rarely satisfied in practice, we aim to: (1) identify the key factors governing BRIDGE's advantage, (2) derive qualitative predictions testable in experiments, and (3) clarify the role of each pipeline component. Our analysis builds on classical learning theory~\cite{shalev2014understanding,vapnik1971uniform} and recent work on teacher-assistant distillation~\cite{mirzadeh2019improvedknowledgedistillationteacher}. Full assumptions and proofs are in Appendix~\ref{app:theory}.

\subsection{Problem Setup}

Consider a black-box teacher $M_T$ that can be queried at most $B$ times, 
and a labeled dataset of $n$ input-label pairs with $n \gg B$. We have three model 
classes with capacities $P_S \ll P_A \ll P_T$:
\[
\mathcal{F}_S \ni M_S, \qquad
\mathcal{F}_A \ni M_A, \qquad
\mathcal{F}_T \ni M_T.
\]
The teacher $M_T$ is only accessible through a black-box API that returns 
a rationale-answer pair $(r_T(x), a_T(x))$ when queried on input $x$. Our 
goal is to train a small student $M_S$ that performs well despite $B \ll n$.

BRIDGE operates in two phases:
\begin{itemize}
    \item \textbf{Phase~1 (Apprenticeship):} Train an assistant $M_A$ on 
    $B$ teacher-labeled examples selected via coverage-aware sampling.
    
    \item \textbf{Phase~2 (Tutoring):} Distill $M_A$ to student $M_S$ 
    using all $n$ inputs, amortizing the limited teacher budget.
\end{itemize}

\subsection{Risk and Approximation Error}

We measure performance using a bounded surrogate loss $\ell_a$ and define 
the \emph{risk} of a predictor $f$ as:
\[
R(f) = \mathbb{E}_{(x,y^*)}\bigl[\ell_a(f(x), y^*)\bigr],
\]
where $y^*$ denotes the ground-truth answer. The \emph{Bayes risk} is 
$R^* = \inf_f R(f)$, representing the best achievable performance.

For a hypothesis class $\mathcal{F}$ and target model $f$, the 
\emph{approximation error} captures how well $\mathcal{F}$ can represent $f$:
\[
\varepsilon_{\text{approx}}(f, \mathcal{F}) 
= \max\Bigl\{0,\; \inf_{h \in \mathcal{F}} R(h) - R(f)\Bigr\}.
\]
This is the irreducible gap between the best model in $\mathcal{F}$ and 
the target-a gap that cannot be closed by more data or better 
optimization, only by increasing model capacity.

\subsection{Multi-Task Distillation Objective}

In BRIDGE, we supervise both rationales and answers using a multi-task 
objective with weight $\beta \in (0, 1]$:
\[
\begin{aligned}
R_g^{(\mathrm{multi})}(f)
&= (1-\beta)\,\mathbb{E}_x\!\bigl[\ell_r\bigl(f(x), r_g(x)\bigr)\bigr] \\
&\quad + \beta\,\mathbb{E}_x\!\bigl[\ell_a\bigl(f(x), a_g(x)\bigr)\bigr].
\end{aligned}
\]

where $g$ is either the teacher $M_T$ or assistant $M_A$. Under mild 
regularity (Assumption~\ref{asmp:translation} in Appendix), the multi-task 
risk remains close to true risk: $|R(f) - R_g^{(\text{multi})}(f)| \le L_\ell \Delta_g$.

\subsection{Key Assumption: Two-Stage Approximation is Easier}

Following the teacher-assistant paradigm~\cite{mirzadeh2019improvedknowledgedistillationteacher}, 
we posit that the intermediate assistant reduces total approximation error:
\[
\begin{aligned}
\varepsilon_{\mathrm{approx}}(M_T,\mathcal{F}_A)
+ \varepsilon_{\mathrm{approx}}(M_A,\mathcal{F}_S)
&\le \varepsilon_{\mathrm{approx}}(M_T,\mathcal{F}_S) \\
&\quad - \Delta_{\mathrm{gap}} .
\end{aligned}
\]
for some gap $\Delta_{\text{gap}} \ge 0$. This assumption is \emph{not derived from first principles}, approximation errors are intractable to compute for neural networks but is motivated by prior work showing intermediate-capacity assistants consistently improve distillation under large capacity gaps~\cite{mirzadeh2019improvedknowledgedistillationteacher,cho2019efficacy}. We validate it indirectly: in Table~\ref{tab:fine-tuning-ta}, the TA trained on 10\% teacher rationales outperforms direct student training on 100\% ground-truth labels by 5--12 points, suggesting $\Delta_{\text{gap}} > 0$ in our settings. The assumption is most plausible when $P_T/P_S > 100\times$ and the task involves compositional reasoning.

\subsection{Coverage-Aware Selection}

To maximize the utility of budget $B$, BRIDGE employs a hybrid selection strategy combining clustering, farthest-point seeding, and difficulty scoring. We acknowledge upfront that the classical $k$-center 2-approximation guarantee~\cite{sener2018coreset} applies only to pure farthest-first traversal; the full hybrid heuristic trades formal guarantees for empirical performance. Here we characterize what \emph{can} be said theoretically.

\paragraph{Selection pipeline.}
BRIDGE partitions the data into $C$ semantic regions via clustering, allocates per-region budgets $\{q_c\}$ proportional to region mass, and within each region ranks examples by a combined score $s(x) = \alpha \bar{D}(x) + (1-\alpha) \bar{n}(x)$, where $\bar{D}(x)$ is the normalized difficulty and $\bar{n}(x)$ is the normalized distance to farthest-point seeds.

\paragraph{What theory guarantees.}
For the \emph{pure diversity} component ($\alpha = 0$), standard results apply: farthest-first traversal achieves a 2-approximation to optimal $k$-center coverage~\cite{gonzalez1985kcenter,sener2018coreset}. Under Lipschitz loss (Assumption~\ref{asmp:regularity}), this translates to a generalization bound with additive term $L \cdot \rho(S_B)$, where $\rho(S_B)$ is the coverage radius.

\paragraph{What theory does not guarantee.}
The difficulty-weighted component ($\alpha > 0$) lacks clean theoretical characterization. Prioritizing high-difficulty examples may improve learning efficiency (akin to importance sampling), but can also increase variance and degrade coverage. We model this trade-off via an additive penalty $\alpha \cdot \Delta_{\text{diff}}$ in our bounds, but $\Delta_{\text{diff}}$ is distribution-dependent and not tightly quantified.

\paragraph{Empirical resolution.}
Given this theoretical ambiguity, we rely on ablations (Table~\ref{tab:perf_sel_learning_ablation}) to validate the hybrid heuristic. Results show $\alpha \approx 0.5$ outperforms both pure coverage ($\alpha = 0$) and pure difficulty ($\alpha = 1$), suggesting the combination captures complementary signals not fully explained by either component alone.

\subsection{Main Result: When BRIDGE Wins}

Under the assumptions above (formalized in Appendix~\ref{app:theory}), we decompose BRIDGE's advantage into interpretable components. Let $R(f)$ denote the risk of model $f$. BRIDGE outperforms direct distillation when:
\[
\Delta_{\text{adv}} 
:= \Delta_{\text{struct}} + \Delta_{\text{sample}} - \Delta_{\text{overhead}} 
> 0.
\]

\paragraph{Structural benefit ($\Delta_{\text{struct}}$).}
This captures the reduced approximation error from using the assistant:
\[
\Delta_{\mathrm{struct}}
= \varepsilon_{\mathrm{approx}}(M_T,\mathcal{F}_S)
 - \varepsilon_{\mathrm{approx}}(M_T,\mathcal{F}_A)
 - \varepsilon_{\mathrm{approx}}(M_A,\mathcal{F}_S).
\]
By our key assumption, $\Delta_{\text{struct}} \ge \Delta_{\text{gap}} \ge 0$. This term is \emph{not directly measurable}, but the empirical gap in Table~\ref{tab:fine-tuning-ta} serves as a proxy.

\paragraph{Sample efficiency benefit ($\Delta_{\text{sample}}$).}
Standard learning theory predicts faster convergence when capacity gaps are smaller:
\[
\Delta_{\text{sample}} = C_{S,T} B^{-\alpha_{S \leftarrow T}} 
- C_{A,T} B^{-\alpha_{A \leftarrow T}} 
- C_{S,A} n^{-\alpha_{S \leftarrow A}},
\]
where $\alpha_{g \leftarrow f} \in (0,1]$ is the convergence exponent and $C_{g,f}$ are distribution-dependent constants. The key qualitative prediction is that when $n \gg B$, the $n^{-\alpha}$ term vanishes and BRIDGE's ability to leverage all $n$ examples becomes decisive.

\paragraph{Overhead cost ($\Delta_{\text{overhead}}$).}
The two-stage pipeline incurs overhead from (1) generalization error on the Phase~2 dataset ($\Gamma_n$) and (2) noise introduced by the assistant's imperfect imitation of the teacher ($L_\ell \Delta_A$):
\[
\Delta_{\text{overhead}} = \Gamma_n + L_\ell \, \Delta_A.
\]
Both terms decrease with more data and better TA training, respectively.

\paragraph{Quantitative limitations.}
The constants $C_{g,f}$, exponents $\alpha_{g \leftarrow f}$, and noise term $\Delta_A$ are not precisely quantified in our analysis, doing so would require strong parametric assumptions on the model classes and distributions. Our contribution is the \emph{decomposition} itself, which identifies the factors determining success and generates testable predictions (Section~\ref{sec:experiments}).

\subsection{Abundant-Data Regime}

The advantage of BRIDGE becomes most pronounced when data are plentiful. 
As $n \to \infty$ with $B$ fixed, Phase~2 overhead vanishes:
\[
\Gamma_n = O\!\left(\sqrt{\tfrac{\log n}{n}}\right) \to 0, \qquad
C_{S,A} \, n^{-\alpha_{S \leftarrow A}} \to 0.
\]
Meanwhile, direct distillation remains bottlenecked by $B$-it cannot 
leverage abundant data without an intermediate model.

Corollary~\ref{cor:asymptotic-advantage} (Appendix) shows that if:
\begin{enumerate}
    \item the two-stage gap $\Delta_{\text{gap}} \ge 0$ (with empirical evidence suggesting $> 0$),
    \item the assistant learns efficiently from the teacher, and
    \item assistant noise $\Delta_A$ is well-controlled (verifiable via held-out evaluation),
\end{enumerate}
then there exists $n_0$ such that for all $n \ge n_0$:
\[
R(S_{\text{BR}}) < R(S_{\text{MTD}}).
\]

\paragraph{Connection to BRIDGE pipeline.}
This theoretical prediction aligns with our empirical design: (1) coverage-aware selection maximizes utility of the $B$ teacher queries; (2) multi-task training on rationales + answers reduces $\Delta_A$ by providing richer supervision than answer-only training; (3) instruction tuning in Phase~2 ensures the student can faithfully learn from TA outputs, further reducing the Phase~2 overhead term.

\subsection{Empirical Testability}

While our analysis does not provide tight numerical bounds, it generates specific \emph{qualitative predictions} that can be falsified experimentally:

\begin{enumerate}[leftmargin=*]
    \item \textbf{Data scaling}: BRIDGE's advantage should \emph{increase} with the ratio $n/B$. Direct distillation cannot leverage data beyond $B$ examples; BRIDGE can. \emph{Prediction}: Performance gap widens as $n$ grows with $B$ fixed. \emph{Test}: Vary dataset size while holding budget constant (Table~\ref{tab: comprehensive-table}).
    
    \item \textbf{Capacity gap sensitivity}: The two-stage approximation advantage ($\Delta_{\text{gap}}$) should be largest when $P_T/P_S$ is extreme and $P_A$ is intermediate. \emph{Prediction}: BRIDGE gains should diminish for smaller capacity gaps. \emph{Test}: Ablate TA size (Table~\ref{tab:fine-tuning-ta}).
    
    \item \textbf{Selection quality}: Coverage-aware selection should outperform random selection by reducing $\rho(S_B)$. \emph{Prediction}: Selective learning $>$ random sampling at equal budget. \emph{Test}: Ablation in Table~\ref{tab:perf_sel_learning_ablation}.
    
    \item \textbf{Rationale supervision}: Multi-task training should reduce $\Delta_A$ by providing richer supervision. \emph{Prediction}: Rationale + answer training $>$ answer-only training. \emph{Test}: Ablation on loss objectives.
\end{enumerate}

All four predictions are validated in Section~\ref{sec:experiments}. The theory's value is not in providing precise numerical bounds which would require intractable distributional assumptions but in identifying the \emph{mechanisms} through which BRIDGE succeeds and generating falsifiable hypotheses.

\subsection{Summary and Limitations}

Our theoretical analysis provides a \emph{risk decomposition} that identifies the key factors determining when BRIDGE outperforms direct distillation:

\begin{enumerate}
    \item \textbf{Abundant data ($n \gg B$):} BRIDGE leverages the full dataset through the assistant, while direct distillation is limited to $B$ examples. \emph{Testable:} Performance should improve with larger $n/B$ ratio (validated in Section~\ref{sec:experiments}).
    
    \item \textbf{Capacity bridging:} The assistant serves as a ``stepping stone'' that reduces effective approximation error. \emph{Testable:} TA-mediated transfer should outperform direct transfer (validated in Table~\ref{tab:fine-tuning-ta}).
    
    \item \textbf{Controlled assistant noise:} Multi-task training on rationales + answers should reduce $\Delta_A$ compared to answer-only training. \emph{Testable:} Ablation in Table~\ref{tab:perf_sel_learning_ablation}.
    
    \item \textbf{Coverage-aware selection:} For pure diversity ($\alpha=0$), formal 2-approximation guarantees apply. For hybrid selection, we rely on empirical validation.
\end{enumerate}

\paragraph{Limitations.}
We are transparent about what our analysis does \emph{not} provide:
\begin{itemize}
    \item \textbf{Tight bounds:} Constants like $C_{g,f}$, $\Delta_A$, and $\Delta_{\text{gap}}$ are not quantified; bounds are qualitative rather than numerically predictive.
    \item \textbf{Novel coverage results:} The coverage guarantees are adapted from known $k$-center results~\cite{gonzalez1985kcenter,sener2018coreset}; we do not claim theoretical novelty for the selection component.
    \item \textbf{Hybrid selection theory:} The difficulty-weighted component lacks rigorous guarantees; its benefit is established empirically, not theoretically.
\end{itemize}

The value of our analysis lies in (1) decomposing BRIDGE's advantage into interpretable factors, (2) generating qualitative predictions validated by experiments, and (3) clarifying which components have formal backing versus empirical justification.

\section{Evaluation Methodology}
\label{sec:eval_methodology}
To evaluate the effectiveness of \tool, we address three research questions spanning overall performance, component-level analysis, and budget sensitivity. Our evaluation focuses on black-box knowledge transfer from an expert teacher to tiny student models, optionally mediated by an intermediate teaching assistant (TA).

\begin{itemize}
    \item \textbf{RQ1: Performance Comparison.}
    How does \tool compare against black-box and white-box distillation baselines when transferring knowledge from black-box teachers to tiny student models?

    \item \textbf{RQ2: Intrinsic Analysis.}
    How do key components (data selection, instruction tuning) and design choices (TA and expert selection) affect knowledge-transfer effectiveness?

    \item \textbf{RQ3: Sensitivity Analysis.}
    How do budget constraints influence TA and student performance?
\end{itemize}

\subsection{Datasets}
We evaluate our method across three distinct domains:  medical, legal, and financial to test generalization across specialized vocabularies and reasoning formats. Table~\ref{tab:dataset_characteristics} summarizes the key statistics.

\begin{itemize}
    \item \textbf{MedConceptsQA (Medical)}~\cite{shoham2024medconceptsqa}:
    A large-scale benchmark designed to test diagnostic reasoning and concept differentiation. It comprises over 800{,}000 question-answer pairs covering specialized terminologies (e.g., ICD-9/10 and ATC drug codes).

    \item \textbf{MMLU-Law (Legal)}~\cite{hendrycks2020mmlu}:
    The Law subset of the MMLU benchmark, focusing on complex legal reasoning. Tasks include interpretation of statutes, case-law analysis, and contract logic.

    \item \textbf{FOMC (Financial)}~\cite{tang2025finmteb}:
    A dataset that classifies central bank communications into three stances: \textit{hawkish}, \textit{dovish}, or \textit{neutral}. The statements are derived from publicly available FOMC meeting minutes, press releases, and transcripts.
\end{itemize}

\begin{table}
    \centering
    \caption{Dataset statistics}
    \label{tab:dataset_characteristics}
    \begin{adjustbox}{max width=\columnwidth}
    \begin{tabular}{@{} l c c c c @{}}
    \hline
    \textbf{Dataset} & \textbf{Domain} & \textbf{Examples} & \textbf{Task Type} & \textbf{Metric} \\
    \hline
    MedConceptsQA & Medical   & 800K+ & MCQ (4-way) & Accuracy \\
    MMLU-Law      & Legal     & 1.5K+ & MCQ (4-way) & Accuracy \\
    FOMC          & Financial & 1K+   & Classification (3-way) & Macro F1 / Acc \\
    \hline
    \end{tabular}
    \end{adjustbox}
\end{table}

\subsection{Experimental Procedure}

\noindent\textbf{RQ1: Performance Comparison.}
We compare \tool against baselines grouped into two categories. All methods are evaluated under comparable access constraints whenever possible.

\begin{itemize}
    \item \textbf{Black-box methods:}
    \begin{itemize}
        \item \textbf{BLKD (Standard Black-box Distillation):}
        Direct fine-tuning on GPT-4-generated labels and rationales, without an intermediate TA.

        \item \textbf{Lion}~\cite{lion}:
        Adversarial distillation using iterative imitation-discrimination generation loops with progressively harder examples.
    \end{itemize}

    \item \textbf{White-box method:}
    \textbf{ULD (Universal Logits Distillation)}~\cite{ULD} uses projection operators to align teacher-student vocabularies, enabling distillation across different tokenizers. Since this approach requires internal access (e.g., logits), we run ULD only for distillation from the TA to the student models.
\end{itemize}

\noindent\textbf{Baseline Scope.}
We restrict comparison to methods operating under identical access constraints (black-box APIs returning text only). Recent approaches such as MiniLLM~\cite{gu2024minillm} and GKD~\cite{agarwal2024gkd} require white-box access to teacher logits or gradients, rendering them inapplicable to our setting. Similarly, methods like PAD~\cite{zhu2024pad} target mid-sized students (3-7B), where capacity gaps are less severe; adapting them to sub-1B models would require substantial modification. Our baselines therefore represent the most directly comparable alternatives under the combined constraints of black-box access, extreme capacity gaps ($>$1000$\times$), and strict query budgets ($<$10\%).

For \tool, we execute the complete two-phase pipeline. Phase~1 trains the TA via selective learning on curated data under a $B\%$ query budget, and Phase~2 applies instruction tuning followed by intensive TA-supervised learning on the full dataset.

In our experiments, we employ GPT-4 as an expert teacher; two mid-sized models (Qwen2.5-7B-Instruct and Mistral-7B-Instruct-v0.3) as TAs for \tool; and three tiny ($<$1B parameters) models (BLOOMZ-560M, OPT-350M, and Pythia-410M) as students.

\noindent\textbf{Implementation Details.}
To ensure reproducibility, we specify all hyperparameters below:

\textit{Selection Pipeline.}
We use a Gaussian Mixture Model (GMM) with $C=16$ components and diagonal covariance for semantic clustering. The GMM is initialized with $k$-means++ centers, trained with 3 EM restarts, and a maximum of 1,000 iterations. Within each cluster, we select $m_0=3$ seed examples via farthest-point sampling for diversity scoring. The difficulty-diversity trade-off parameter is set to $\alpha=0.5$ throughout all experiments. The warm-up stage uses 10\% of the dataset.

\textit{Instruction Tuning.}
For the instruction alignment step in Phase~2, we use \texttt{databricks-dolly-15k}~\cite{DatabsricksDolly}, a publicly available instruction-following dataset containing 15,000 high-quality prompt-response pairs across diverse tasks (brainstorming, classification, QA, summarization, etc.). We fine-tune students on the full dataset for 5 epochs using LoRA~\cite{hu2022lora} (rank $r=16$, $\alpha=32$, dropout 0.05) to reduce memory overhead.

\textit{Training Configuration.}
For both TA apprenticeship (Phase~1) and student distillation (Phase~2), we use the AdamW optimizer with learning rate $2\times10^{-5}$, linear warmup over 10\% of steps, and cosine decay. Batch size is 8 for TAs and 16 for students. The multi-task loss weight is $\beta=0.5$, equally balancing rationale and answer supervision. All models are trained in bfloat16 precision on NVIDIA A100 GPUs. TA training takes approximately 10 hours per dataset; student training takes 2-4 hours depending on model size.
For detailed implementation and other hyper-parameters, one can see our website~\cite{website}.

\noindent\textbf{RQ2: Intrinsic Analysis.}
To quantify the contributions of specific components and design choices, we conduct controlled studies where we modify one factor at a time while keeping the remaining settings fixed. This isolates the causal impact of each component on downstream performance.

\begin{itemize}

    \item \textbf{TA Apprenticeship Efficacy:} We interrogate the necessity of the Phase 1 apprenticeship by comparing \tool against a Naive Fine-Tuning baseline. This setup eliminates the expert teacher entirely, fine-tuning the TA directly on the full dataset using ground-truth labels to determine if the value lies in the rationales or merely in the data coverage.

    \item \textbf{Mechanism Ablation:}  We isolate the effects of our two core methodological contributions: (i) \textit{Selection Strategy:} Replacing the coverage-guided pipeline (difficulty + diversity) with uniform random sampling under the same budget $B$. (ii) \textit{Instruction Alignment:} Removing the Phase 2 instruction tuning step to quantify the impact of ``instruction blindness'' on distillation fidelity.
    
    \item \textbf{TA Selection:} We systematically vary the capacity of the intermediate bridge using the Qwen2.5 family (3B, 7B, 14B, 32B) to identify the optimal balance between transfer fidelity and computational overhead.

    \item \textbf{Teacher Quality:}
    We evaluate three expert teachers (GPT-4, Gemini-2.5, and DeepSeek-V3) with a fixed TA to assess how expert quality affects knowledge transfer to the TA, and consequently, to students.
\end{itemize}

\noindent\textbf{RQ3: Sensitivity Analysis.}
We assess the framework's and scalability and stability by controlling two key variables:

\begin{itemize}
    \item \textbf{Student Scalability:}  We fix a TA and vary the student size up to 1B parameters to study scaling behavior within the ``tiny model'' regime and to provide an approximate upper bound for this setting.

    \item \textbf{Data Budget: } We analyze robustness by varying the Phase~1 query budget $B$ across four ranges: Low (5\%), Default (10\%), Moderate (20\%), and High (50-70\%). We report the resulting trade-off between API cost and downstream performance for both the TA and the student. Additionally, we fix the TA and teacher configurations and scale the student model up to 1B parameters (testing Llama-3.2-1B) to determine if the framework's benefits extend beyond the tiny regime and to test its efficacy on higher-capacity architectures.
\end{itemize}

\subsection{Evaluation Metrics}
Evaluating generative models on fixed-choice tasks requires robust parsing of free-form outputs~\cite{wang2024myanswercfirsttoken}. To ensure consistent and fair evaluation, we adopt an LLM-as-a-judge approach to extract each model's selected answer (and, when present, its rationale) from the decoded response text. We employ GPT-5~\cite{openai2025_gpt5_system_card} to normalize student outputs into discrete labels for scoring. We then compute downstream task performance on held-out test sets for each benchmark:

\begin{itemize}
    \item \textbf{MedConceptsQA \& MMLU-Law.}
    Top-1 multiple-choice accuracy, defined as the fraction of test examples where the predicted answer matches the ground-truth label:
    \begin{equation}
        \text{Accuracy} = \frac{1}{N}\sum_{i=1}^{N} \mathbf{1}\!\left(\hat{y}_i = y_i\right).
    \end{equation}

    \item \textbf{FOMC.}
    Macro F1 score and accuracy for three-way classification. Macro F1 is computed as the unweighted average of per-class F1 scores:
    \begin{align}
        \text{MacroF1} &= \frac{1}{K} \sum_{c=1}^{K} \text{F1}_c, \\
        \text{F1}_c &= \frac{2\,\text{Precision}_c \cdot \text{Recall}_c}{\text{Precision}_c + \text{Recall}_c},
    \end{align}
    where
    \begin{align}
        \text{Precision}_c &= \frac{TP_c}{TP_c + FP_c}, \\
        \text{Recall}_c &= \frac{TP_c}{TP_c + FN_c}.
    \end{align}
\end{itemize}

\section{Experimental Results}
\label{sec:experiments}

\subsection{Performance Comparison}

\begin{table*}\centering
\caption{Performance Comparison of \tool across all baselines}
\label{tab: comprehensive-table}
\resizebox{\textwidth}{!} {
\begin{tabular}{@{}lllllccc c@{}}\toprule
\textbf{Category} & \textbf{Approach} & \textbf{Teacher} & \textbf{TA} & \textbf{Student} &\textbf{MedConceptsQA} &\textbf{ MMLU Prof Law} &\textbf{FOMC } &\textbf{Average}   \\\cmidrule{1-9}
\multirow{6}{*}{\textbf{Original}} &\multirow{6}{*}{\textbf{None}} &\multirow{6}{*}{None} &\multirow{6}{*}{None} &GPT-4 &94.27 &87.70 & 71.00  & 84.32 \\\cmidrule{5-9}
& & & &Qwen2.5-7B-Instr. &58.00 &51.96 &55.33 & 55.10 \\\cmidrule{5-9}
& & & &Mistral-7B-Instr.-v0.3 &36.40 &44.98 &45.59 & 42.32 \\\cmidrule{5-9}
& & & & Bloomz-560M & 24.60 & 25.68 & 22.19 & 24.16\\\cmidrule{5-9}
& & & & OPT-350M& 23.90& 24.9& 22.85 & 23.88\\\cmidrule{5-9}
& & & & Pythia-410M& 25.43& 25.10& 30.13 & 26.89\\\cmidrule{1-9}
\multirow[c]{6}{*}{\textbf{Whitebox }} &\multirow[c]{6}{*}{\textbf{ULD}} &\multirow[c]{6}{*}{None} &\multirow[c]{3}{*}{Qwen2.5-7B-Instr.} & Bloomz-560M& 25.80& 24.92& 23.50 & 24.74 \\\cmidrule{5-9}
& & & &OPT-350M&24.90&25.55& 23.80 & 24.75\\\cmidrule{5-9}
& & & &Pythia-410M&26.50&25.65& 31.80 & 27.98\\\cmidrule{4-9}
& & &\multirow[c]{3}{*}{Mistral-7B-Instr.-v0.3} &Bloomz-560M&23.50&27.60& 23.30 & 24.80 \\\cmidrule{5-9}
& & & &OPT-350M&24.00&26.40& 24.30 & 24.90 \\\cmidrule{5-9}
& & & &Pythia-410M&25.40&25.50& 30.90 & 27.26\\\cmidrule{1-9}
\multirow{14}{*}{\textbf{Blackbox}} &\multirow{3}{*}{\textbf{BLKD}} &\multirow{3}{*}{GPT-4} &\multirow{3}{*}{None} &Bloomz-560M&27.00&27.70&28.93 & 27.87\\\cmidrule{5-9}
& & & &OPT-350M&26.20&27.70&26.40 & 26.77\\\cmidrule{5-9}
& & & &Pythia-410M&28.00&24.60&26.65 & 26.42\\\cmidrule{2-9}
&\multirow{3}{*}{\textbf{Lion}} &\multirow{3}{*}{GPT-4} &\multirow{3}{*}{None} &Bloomz-560M&25.00&24.30&23.40 &24.20\\\cmidrule{5-9}
& & & &OPT-350M&19.40&20.50&23.40 & 21.10\\\cmidrule{5-9}
& & & &Pythia-410M&20.60&23.20&18.50 & 20.77\\\cmidrule{2-9}
&\multirow{8}{*}{\textbf{\tool}} &\multirow{8}{*}{GPT-4} &\multirow{2}{*}{None} &Qwen2.5-7B-Instr. &86.12 &65.78 &70.80 & 74.22 \\\cmidrule{5-9}
& & & &Mistral-7B-Instr.-v0.3 &56.66 &56.66 &60.50 &57.94 \\\cmidrule{4-9}
& & &\multirow{3}{*}{Qwen2.5-7B-Instr.} &\textbf{Bloomz-560M}&\textbf{31.20}&\textbf{33.59}&\textbf{33.19}&\textbf{32.66} \\\cmidrule{5-9}
& & & &\textbf{OPT-350M}&\textbf{33.20}&\textbf{33.75}&\textbf{34.50}&\textbf{33.82}\\\cmidrule{5-9}
& & & &\textbf{Pythia-410M}&\textbf{35.60}&\textbf{28.11}&\textbf{35.60}&\textbf{33.10}\\\cmidrule{4-9}
& & &\multirow{3}{*}{Mistral-7B-Instr.-v0.3} &\textbf{Bloomz-560M}&\textbf{31.20}&\textbf{30.81}&\textbf{32.90}&\textbf{31.64}\\\cmidrule{5-9}
& & & &\textbf{OPT-350M}&\textbf{30.80}&\textbf{31.22}&\textbf{31.40}&\textbf{31.14}\\\cmidrule{5-9}
& & & &\textbf{Pythia-410M}&\textbf{31.20}&\textbf{34.22}&\textbf{36.50}&\textbf{33.97}\\
\bottomrule
\end{tabular}
}
\end{table*}

Table~\ref{tab: comprehensive-table} shows the performance of \tool compared to representative black-box (BLKD, Lion) and white-box (ULD) methods across three benchmarks: MedConceptsQA, MMLU-Law, and FOMC. Overall, \tool consistently outperforms all baselines, achieving substantial improvements over both black-box approaches and white-box distillation.

Specifically, \tool enables tiny models to transcend their capacity limitations, delivering relative performance improvements of 23-41\% over their pre-distillation baselines. This suggests that the framework successfully internalizes key teacher capabilities despite the extreme parameter disparity.

Beyond metric gains, \tool also yields qualitative improvements in student outputs, producing more coherent responses aligned with teacher-style reasoning. Figure~\ref{fig:examples} provides a representative example showing how \tool transforms a tiny student model (BLOOMZ-560M) from producing ambiguous, nonsensical outputs into generating correct answers with grounded reasoning on a financial sentiment classification task.

\begin{figure}
    \centering
    \includegraphics[width=1.0\linewidth]{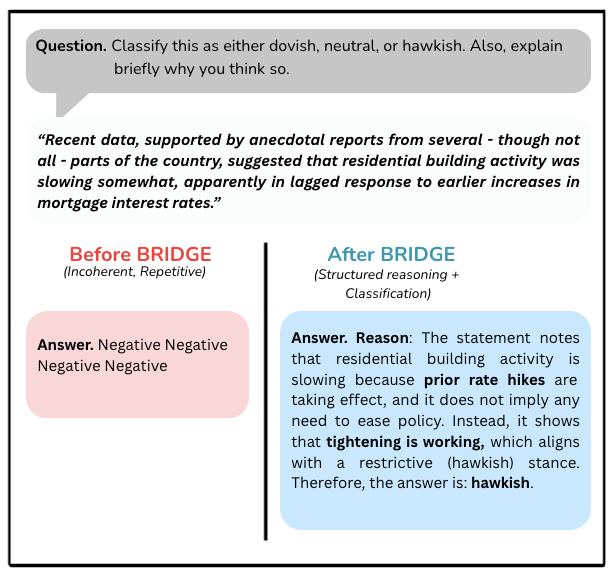}
    \caption{A representative example from the FOMC benchmark showing the BLOOMZ-560M student's output before (hallucination) and after (grounded reasoning) \tool distillation.}
    \label{fig:examples}
\end{figure}

\textbf{Comparison against Black-box methods.}
As shown in Table~\ref{tab: comprehensive-table}, \tool consistently surpasses the direct distillation baseline (BLKD) across all domains, with average approximate gains on students of about 17-26\% over pure black box distillation. 
Notably, \tool strongly outperforms Lion across all benchmarks, with average student gains of roughly 35-59\%. 
In contrast to \tool, Lion's adversarial data generation even hurts performance in some cases (Pythia-410M), indicating that adversarial training without selective learning or intermediate bridging capacity can be harmful for tiny models.

To rigorously determine whether Lion's failure stems from the capacity gap or its methodology, we conducted a controlled experiment applying Lion with Intermediate TA. In this configuration, we inject our 7B Teacher Assistant into the Lion pipeline, mirroring \tool's two-phase architecture but retaining Lion's adversarial logic.

\begin{table}
\centering
\small
\caption{Extra comparison for Lion with an intermediate (Qwen2.5-7B-Instr. TA) model on FOMC}
\label{tab:lion_intermediate}
\begin{tabular}{@{} l 
                 c c c c  @{}}
\toprule
\textbf{Student}
  & \multicolumn{1}{c}{\textbf{Baseline }}
  & \multicolumn{1}{c}{\textbf{Lion}}
  & \multicolumn{1}{c}{\textbf{Lion + TA}}
  & \multicolumn{1}{c}{\textbf{\tool}} \\
\midrule
%  Qwen2.5-7B & 55.33 & {-} & 60.73 & 70.8 \\
% \midrule
 BLOOMZ-560M    & 22.2  & 23.4 &  23.3 & 33.2\\
 OPT-350M       & 22.9 & 23.4 & 19.3  & 34.5 \\
 Pythia-410M    & 30.1 & 18.5 & 26.5 & 35.6 \\
\midrule
\textbf{Average} & 25.1 & 21.8 & 23.1 & 34.4 \\
\bottomrule
\end{tabular}
\end{table}

As detailed in Table~\ref{tab:lion_intermediate}, incorporating an intermediate TA with Lion does improve the student models by 5.97\% on average. However, despite this structural enhancement, Lion remains unstable at the student level. For instance, the performance of OPT-350M degrades by 17\%. In contrast, \tool delivers consistent improvements across all three students. Furthermore, Lion with TA is still dramatically underperformed compared to \tool, trailing by an average of 32.9\% across students and 16.58\% at the TA level (Table~\ref{tab:lion_intermediate}). 
Since both methods utilize the same intermediate capacity (TA), the performance delta is attributable strictly to methodology. \tool's combination of \textit{selective sampling} and \textit{instruction-first curriculum} proves fundamentally superior to unstructured adversarial generation for the training of tiny models.

% Given that both methods use the same two-phase architecture and the same 7B intermediate model, this performance gap points out \tool's effective methodological choices. The persistent underperformance and instability of adversarial samples even when properly staged through an intermediate model suggest that adversarial sample generation without proper selective data or rationale-based guidance remains fundamentally problematic for tiny models. In particular, the mechanism ablation in Section~\ref{sec:experiments} shows that selective sampling and instruction tuning are all needed for effective extreme distillation, whereas Lion's adversarial data generation lacks this structured pedagogy. Furthermore, the iterative adversarial process appears to generate samples that, though challenging, miss the pedagogical structure necessary for effective knowledge transfer to sub-1B parameter models.

\textbf{Comparison against White-box methods.}
As seen in Table~\ref{tab: comprehensive-table}, \tool outperforms ULD~\cite{ULD} across all benchmarks, with average approximate student gains of about 18-33\% relatively.
% , despite ULD having white-box access to the TA logits.
These results challenge the prevailing assumption that access to internal logits inherently yields better distillation than text-only signals. 
We attribute this to two key factors:
\begin{itemize}
    \item \textbf{Capacity-aware transfer.}
    ULD forces projection of high-dimensional logits onto a severely constrained student. Under $>$1000$\times$ capacity gaps, this projection is inherently lossy. In contrast, \tool's two-phase architecture decomposes the gap through an intermediate TA, while selective learning ensures pedagogically optimal example ordering.

    \item \textbf{Process vs.\ distribution.}
    \tool transfers \textit{reasoning processes} via rationale supervision and ensures tiny students can parse them through instruction alignment. ULD transfers only output distributions, assuming students can already interpret task semantics---an assumption that fails for sub-billion-parameter models.
\end{itemize}

These substantial gaps between \tool and BLKD; \tool and Lion; or \tool and ULD demonstrate that \tool's structured two-phase approach with intermediate capacity bridging is fundamentally superior to direct fine-tuning, adversarial distillation, or cross-tokenizer whitebox distillation for tiny student models.

To verify that our gains are not an artifact of specific model pairing (Qwen $\to$ Student), we replicated the pipeline using Mistral-7B-Instruct-v0.3 as the TA. Our results confirm that the magnitude of improvement remains consistent, validating that the framework is agnostic to the specific intermediate architecture.

\begin{tcolorbox}[ 
    enhanced,
    unbreakable, 
    title={Answer to RQ1}, 
    colback=white, colframe=gray!50!black, boxrule=0.6pt, rounded corners, coltitle=white, fonttitle=\bfseries, boxed title style={ colback=gray!70!black, rounded corners, boxrule=0pt, left=4pt, right=4pt, top=2pt, bottom=2pt } ] 

\tool consistently outperforms both black-box and white-box baselines, achieving gains of up to 41\%. Our analysis confirms that for tiny student models, structured rationale-based supervision is significantly more effective than adversarial hardness or raw logit matching. The framework successfully bridges the extreme capacity gap, enabling sub-1B parameter models to internalize reasoning patterns previously accessible only to large experts. 

\end{tcolorbox}

\subsection{{Intrinsic Analysis}}

To understand the driving factors of \tool's performance, we conduct a series of ablation to isolate the contribution of each key component in \tool's design, ranging from the distillation mechanisms to the choice of teacher, TA and student architectures.

\subsubsection{TA Apprenticeship Efficacy}

\begin{table}
\centering
\small
\caption{Performance comparison between Naive TA Fine-Tuning (using 100\% ground-truth labels) and \tool's Apprenticeship (using 10\% teacher rationales)}
\label{tab:fine-tuning-ta}
\begin{tabular}{@{} l l c c c @{}}
\toprule
 & \textbf{Model Name} & \multicolumn{1}{c}{\textbf{Before }} & \multicolumn{1}{c}{\textbf{Naive FT}}  & \multicolumn{1}{c}{\textbf{BRIDGE }}  \\
\midrule
Teacher & GPT-4 & 87.70 & 87.70 & 87.70  \\
\midrule
TA & Qwen2.5-3B & 43.22 & 44.10 & 54.38 \\
\cmidrule{2-5}
\multirow{3}{*}{Student} 
 & BLOOMZ-560M & 25.68 & 26.10 & 29.33  \\
 & OPT-350M & 24.97 & 25.02 & 28.75 \\
 & Pythia-410M & 25.1 & 25.90 & 27.36 \\
\midrule
TA & Qwen2.5-7B-Instr. & 51.96 & 53.19 & 65.78 \\
\cmidrule{2-5}
\multirow{3}{*}{Student} 
 & BLOOMZ-560M & 25.68 & 26.39 & 33.59 \\
 & OPT-350M & 24.97 & 25.95 & 33.75 \\
 & Pythia-410M & 25.10 & 26.35 & 28.11 \\
\midrule
TA & Qwen-14B & 59.74 & 60.32 & 73.15 \\
\cmidrule{2-5}
\multirow{3}{*}{Student} 
 & BLOOMZ-560M & 25.68 & 27.54 & 38.29  \\
 & OPT-350M & 24.97 & 23.48 & 36.71 \\
 & Pythia-410M & 25.10 & 24.53 & 29.64 \\

\bottomrule
\end{tabular}
\end{table}

To validate the necessity of the Phase 1 (Budget-Aware TA Apprenticeship), we interrogate a critical counter-factual: could the TA serve as an effective bridge simply by being fine-tuned on the full dataset using ground-truth labels? This ``Naive Fine-Tuning'' baseline represents a zero-cost alternative that eliminates the expert teacher entirely while preserving the two-stage architecture.
In this experiment, we fine-tuned Qwen2.5 TAs of varying capacities (3B, 7B, 14B) on the full MMLU-Law training set using standard supervised learning. 
We then apply the same knowledge-transfer procedure from the fine-tuned TA to the student models using a multi-task objective. 

Table~\ref{tab:fine-tuning-ta} shows that naive fine-tuning improves the performance of TAs by a marginal gain less than 2.3\% across model sizes. Consequently, the downstream signal is weak: students' performance gains are  less than 4\%. Especially, under 14B TA, the performance of OPT-350M and Pythia-410M is even degraded. 
In contrast, \tool with Budget-Aware TA Apprenticeship delivers substantially larger gains: the same 7B TA improves by 26.6\% and yields student improvements of 12-35\%. 
\tool outperforms the baseline using only small amount of the data (e.g., 10\%), proving that expert-derived Chain-of-Thought rationales provide a far richer learning signal than 100\% coverage of ground-truth labels. 
Furthermore, the failure of the larger 14B TA to improve under direct fine-tuning isolates the bottleneck as supervision quality rather than model capacity. 
Overall, these results confirms that the Teacher's value lies not in providing correct answers, but in demonstrating the deductive processes necessary for effective downstream distillation.

\subsubsection{{Mechanism Ablation.}}
We first evaluate the contributions of our two core phases: Coverage-guided Selection in Phase 1 and Instruction Alignment in Phase 2.

\textbf{Impact of Coverage-guided Selection.}
We first test the hypothesis that \textit{data quality outweighs quantity} under strict budget constraints. We replace our coverage-guided selection pipeline (difficulty + diversity) with uniform random sampling while maintaining the exact same query budget ($B$).
As seen in Table \ref{tab:perf_sel_learning_ablation}, removing the Coverage-guided Selection pipeline and replacing it by random sampling degrades students' performance by 16.66\% on average. 
Especially, Pythia-410M's performance suffers a significant decline of about 34\%, from 35.60\% to 23.50\%. 
These results confirms that within the capacity-budget trap, the difficulty and diversity (information density) are essential of transfer success. Random sampling fails to capture the frontier examples necessary to stretch the TAs' capacities, leading to a weaker supervisor and consequently, a collapsed student.

% \begin{table*}
% \centering
% \small
% \caption{Performance of \tool (Qwen2.5-7B-Instr. as TA and GPT-4 as Teacher) on MedConceptsQA: without vs with Selective Learning (SL)}
% \label{tab:perf_sel_learning_ablation}
% \begin{tabular}{@{} l l l c c @{}}
% \toprule
% \textbf{Setting} & \textbf{Model Name} & \multicolumn{1}{c}{\textbf{Before}} & \multicolumn{1}{c}{\textbf{After (\tool)}} \\
% \midrule
% \multirow{4}{*}{Without SL}
% & BLOOMZ-560M & 24.60 & 30.90   \\
% &  & OPT-350M    & 23.90 & 29.00    \\
% &  & Pythia-410M & 25.43 & 23.50   \\
% \midrule
% \multirow{4}{*}{With SL}
% & BLOOMZ-560M & 24.60 & 31.20 \\
% &  & OPT-350M    & 23.90 & 33.20 \\
% &  & Pythia-410M & 25.43 & 35.60 \\
% \bottomrule
% \end{tabular}
% \end{table*}

\begin{table}
\centering
\small
\caption{Performance of \tool (Qwen2.5-7B-Instr. as TA and GPT-4 as Teacher) on MedConceptsQA: without vs with Coverage-guided Selection (CS)}
\label{tab:perf_sel_learning_ablation}
\begin{tabular}{@{} l l c c @{}}
\toprule
\textbf{Configuration} & \textbf{Student Model} & \textbf{Baseline} & \textbf{\tool} \\
\midrule
\multirow{3}{*}{Without CS}
  & BLOOMZ-560M & 24.60 & 30.90 \\
  & OPT-350M    & 23.90 & 29.00 \\
  & Pythia-410M & 25.43 & 23.50 \\
\midrule
\multirow{3}{*}{With CS}
  & BLOOMZ-560M & 24.60 & \textbf{31.20} \\
  & OPT-350M    & 23.90 & \textbf{33.20} \\
  & Pythia-410M & 25.43 & \textbf{35.60} \\
\bottomrule
\end{tabular}
\end{table}

\textbf{Impact of Instruction Alignment.}
We also evaluate the impact of Instruction Alignment. We compare the full \tool's pipeline against its variant where the tiny student is exposed directly to the TA's rationale-heavy supervision without instruction alignment. 
As shown in Table~\ref{tab:instruct_tuning_ablation_mmlu}, Instruction Alignment yields substantially larger gains for all students, providing on average an improvement by 22.85\%  over the variant without instruction alignment. 
This confirms our premise regarding the instruction blindness of sub-billion parameter models: without the behavioral prerequisite of instruction alignment, the student lacks the syntactic competence to parse and internalize the complex Chain-of-thought supervision provided by the TA, rendering the distillation signal ineffective.

\begin{table}
\centering
\small
\caption{Performance on MedConceptsQA: without vs with Instruction Alignment (IA)}
\label{tab:instruct_tuning_ablation_mmlu}
\begin{adjustbox}{max width=\columnwidth}
\begin{tabular}{@{} l l c c @{}}
\toprule
\textbf{Setting} & \textbf{Model Name}
  & \multicolumn{1}{c}{\textbf{Baseline}} 
  & \multicolumn{1}{c}{\textbf{\tool}} \\
\midrule
\multirow{3}{*}{Without IA} 
 & BLOOMZ-560M & 24.60 & 26.70 \\
 & OPT-350M    & 23.90 & 28.70 \\
 & Pythia-410M & 25.43 & 26.00 \\
\midrule
\multirow{3}{*}{With IA} 
 & BLOOMZ-560M & 24.60 & \textbf{33.20} \\
 & OPT-350M    & 23.90 & \textbf{35.60} \\
 & Pythia-410M & 25.43 & \textbf{31.20} \\
\bottomrule
\end{tabular}
\end{adjustbox}
\end{table}

\subsubsection{TA Capacity Scaling Analysis}

\begin{figure}
\centering
\includegraphics[width=1\linewidth]{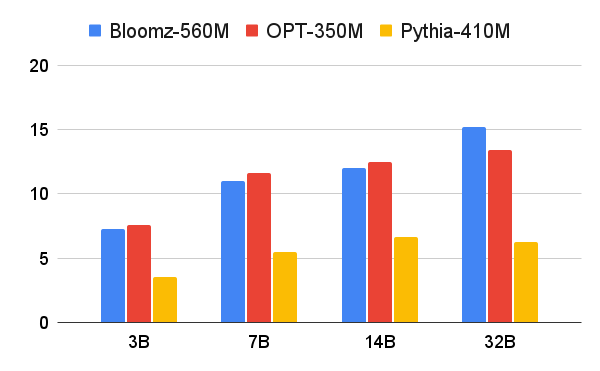}
\caption{Performance of \tool with Qwen2.5 series models in different sizes as the TA on FOMC (GPT-4 as the Teacher)}
\label{fig:ta_analysis}
\end{figure}

To determine the optimal scale for the intermediate bridge, we evaluate the impact of Teacher Assistant capacity on downstream transfer fidelity. We systematically vary the Qwen2.5 TA size across four orders of magnitude (3B, 7B, 14B, 32B) on the FOMC benchmark.

In Figure~\ref{fig:ta_analysis}, there is a strong positive correlation between TA capacity and student efficacy. Upgrading the intermediate model from the smallest variant (3B) to the mid-sized variant (14B) drives a near two-fold increase in student performance, confirming that a more capable bridge generally constructs a higher-fidelity learning signal.
The transition from 3B to 7B yields the steepest marginal improvement (6-13\%), followed by moderate gains at 14B (9-16\%). However, we observe a saturation point beyond this threshold. The 32B TA gives very large gains to BLOOMZ-560M and OPT-350M, while Pythia-410M suffers a performance regression compared to the 14B baseline. 
These results could suggest that the 7B parameter TA offers a good balance, maximizing transfer fidelity without incurring the diminishing returns of larger intermediate models.

\subsubsection{{Teacher Quality}}
To verify that \tool's efficacy is not an artifact of a specific teacher's latent distribution, we evaluate the framework's robustness across three distinct frontier models: GPT-4, Gemini-1.5-Pro, and DeepSeek-V3,

The results in Table \ref{tab:expert_selection_mmlu_law} show a stable downstream trend, confirming that \tool's improvements are not tied to a specific teacher. 
GPT-4 yields slightly higher absolute performance, but Gemini-2.5 achieves similar student improvements (about 12-33\%), implying robustness of the approach.
Interestingly, DeepSeek-V3 shows lower TA improvement (+4.13\%) but maintains strong student gains (13-30\%), suggesting factors beyond raw expert performance (such as reasoning style) influence knowledge transfer quality.

\begin{table*}
\centering
\small
\caption{Performance of \tool (Qwen2.5-7B-Instr. as the TA) by different Teacher on MMLU LAW}
\label{tab:expert_selection_mmlu_law}
\begin{tabular}{@{} l l l c c @{}}
\toprule
\textbf{Teacher} &  & \textbf{Model Name} & \multicolumn{1}{c}{\textbf{Baseline}} & \multicolumn{1}{c}{\textbf{\tool}} \\
\midrule
\multirow{5}{*}{GPT-4}
 & Teacher & GPT-4 & 87.70 & 87.70  \\
 & TA & Qwen2.5-7B-Instr. & 58.00 & 65.78 \\
% \cmidrule{2-5}
 
\cmidrule{2-5}
 & \multirow{3}{*}{Student} & BLOOMZ-560M & 25.68 & 33.59 \\
 & & OPT-350M & 24.97 & 33.75 \\
 & & Pythia-410M & 25.10 & 28.11 \\
\midrule
\multirow{5}{*}{Gemini-2.5}
 & Teacher & Gemini-2.5 & 88.00 & 88.00 \\
 & TA & Qwen2.5-7B-Instr. & 58.00 & 63.20  \\
% \cmidrule{2-5}
 
\cmidrule{2-5}
 & \multirow{3}{*}{Student} & BLOOMZ-560M & 25.68 & 33.20  \\
 & & OPT-350M & 24.97 & 33.40  \\
 & & Pythia-410M & 25.10 & 28.60  \\
\midrule
\multirow{5}{*}{DeepSeek-V3}
 & Teacher & DeepSeek-V3 & 88.50 & 88.50  \\
 & TA & Qwen2.5-7B-Instr. & 58.00 & 60.40 \\
% \cmidrule{2-5}
 
\cmidrule{2-5}
 & \multirow{3}{*}{Student} & BLOOMZ-560M & 25.68 & 32.50  \\
 & & OPT-350M & 24.97 & 32.70  \\
 & & Pythia-410M & 25.10 & 28.50  \\
\bottomrule
\end{tabular}
\end{table*}

\begin{tcolorbox}[ 
    enhanced,
    unbreakable, 
    title={Answer to RQ2}, 
    colback=white, colframe=gray!50!black, boxrule=0.6pt, rounded corners, coltitle=white, fonttitle=\bfseries, boxed title style={ colback=gray!70!black, rounded corners, boxrule=0pt, left=4pt, right=4pt, top=2pt, bottom=2pt } ]
    
We demonstrate that Budget-Aware TA Apprenticeship is far superior to direct fine-tuning, proving that reasoning transfer is the primary driver of success. 
Additionally, our intrinsic analysis confirms that Coverage-guided Selection and Instruction Alignment are structural prerequisites for distilling knowledge to sub-billion parameter models. 
We identify that a 7B parameter TA offers the optimal balance for distilling to tiny models, as larger TAs yield diminishing returns. 
\end{tcolorbox}

\subsection{Sensitivity Analysis}

\subsubsection{Impact of Student Scalability}

\begin{figure}
\centering
\includegraphics[width=1\linewidth]{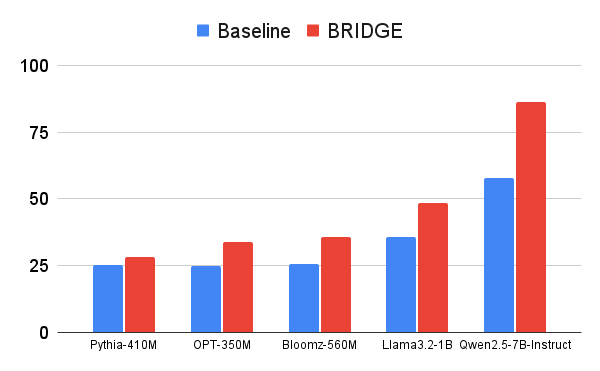}
\caption{Performance of \tool (Qwen2.5-7B-Instr. as the TA and GPT-4 as the Teacher) with the students in various size on MMLU Prof Law}
\label{fig:student_change}
\end{figure}

To assess whether \tool's benefits are confined to the models ($\le$500M parameters) or extend to higher-capacity students, we evaluate the framework on Llama-3.2-1B. We hypothesize that a student with greater representational capacity can more effectively internalize the dense supervision provided by the TA, leading to superior knowledge retention.

The results in Figure \ref{fig:student_change} strongly confirm our hypothesis.
Llama-3.2-1B student achieves a significant performance gain of 35.17\% in performance, outperforming the three smaller student models in relative improvements.
This scaling behavior confirms that \tool is not merely a patch for low-capacity models but a robust distillation mechanism that scales efficiently with student capacity.

\subsubsection{Impact of Data Budget}

To quantify the sensitivity of \tool to resource constraints, we systematically vary the Phase 1 query budget $B$ across five distinct regimes, ranging from extremely scarce (5\%) to abundant (70\%). This analysis aims to delineate the trade-off between supervision volume and downstream performance, identifying the minimal resource threshold required for effective distillation.

% Figure~\ref{fig:budget-perf} shows that \tool scales well with data availability for both TA and students, but with diminishing returns beyond a 50\% budget. 
%
%
As shown in Figure~\ref{fig:budget-perf}, even at only 5\% data, \tool already surpasses BLKD at 100\% of the dataset (30.33\% vs.\ 27.02\% student on average Table~\ref{tab: comprehensive-table}). This finding underscores the high quality of our selective sampling mechanism, which successfully identifies the most pedagogical examples. Consequently, the 10\% budget setting offers a strong practical trade-off. At this level, TA accuracy rises to 86.12\% and student improvement yields an average by 35.25\%. These trends confirm \tool's robustness under tight data budgets, a key requirement for real-world deployment

\begin{figure}
\centering
\includegraphics[width=1\linewidth]{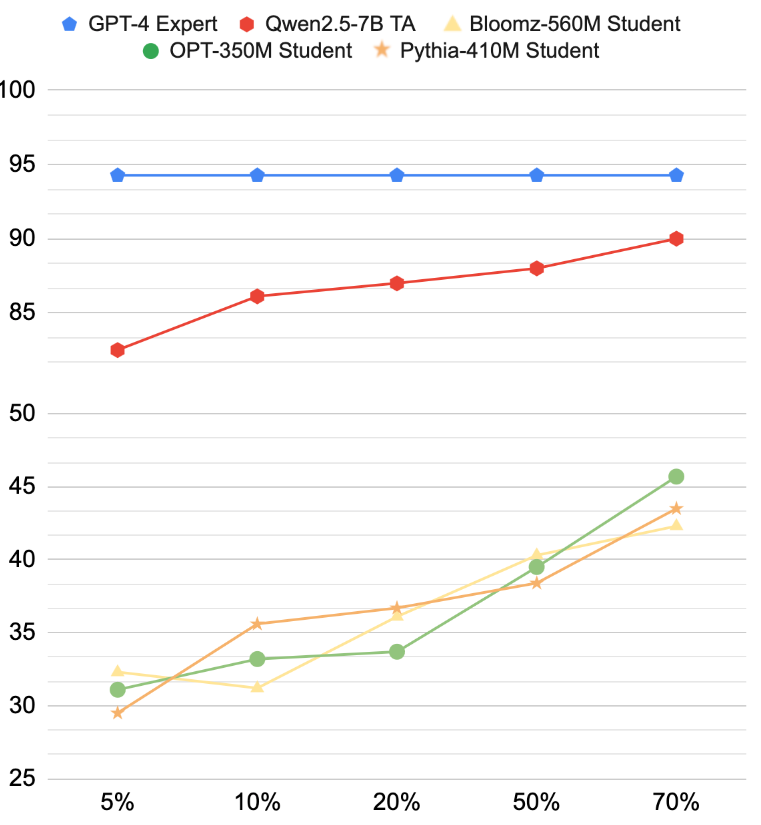}
\caption{Performance of BRIDGE by data budgets on MedconceptsQA}
\label{fig:budget-perf}
\end{figure}

The scaling behavior reveals a crucial decoupling between the TA and the students. The performance jumps significantly from 5\% to 10\% (reaching $\sim$86\%), but subsequently plateaus, requiring massive data increases (20\% $\to$ 70\%) to increase 4 points of accuracy.
Conversely, students' performance accelerates as data volume increases. The steepest growth often occurs between 50\% and 70\%, suggesting that tiny models benefit disproportionately as the TA approaches expert-level parity ($\ge$90\%).

Additionally, while 10\% could remain the efficiency optimal, capturing the steepest portion of the TA's learning curve, the results indicate that higher budgets are not wasted. Thus, for applications demanding maximum student performance, a budget of 50-70\% is justifiable, whereas for cost-sensitive deployments, the 10\% setting offers a good balance.

\begin{tcolorbox}[
    enhanced, 
    unbreakable, 
    title={Answer to RQ3}, 
    colback=white, colframe=gray!50!black, boxrule=0.6pt, rounded corners, coltitle=white, fonttitle=\bfseries, boxed title style={ colback=gray!70!black, rounded corners, boxrule=0pt, left=4pt, right=4pt, top=2pt, bottom=2pt } ] 
\tool is robust to the choice of expert teacher and scales effectively to larger student models (1B parameters).
Additionally, \tool has been shown highly data-efficient and robust to budget constraints. 
% It outperforms standard distillation baselines (which use 100\% of the data) while using only 3-5\% of the expert teacher budget. 
While performance continues to scale with larger budgets, the 10\% regime offers the optimal balance for deployment, delivering high accuracy gains at a fraction of the cost of full-scale distillation.

\end{tcolorbox}
\subsection{Threats to Validity}

We identify potential threats to the validity of our study categorized into internal, construct, and external validity, along with the mitigation strategies employed.

\textbf{Internal Validity.}
A primary threat to internal validity is the potential bias in our Teacher Assistant (TA) selection. If the chosen TA architecture (Qwen2.5-7B) happens to be uniquely compatible with the specific teacher (GPT-4) or student architectures, our results might not reflect the framework's true efficacy. To mitigate this, we conducted ablation studies using a completely different TA architecture (Mistral-7B-Instruct-v0.3) and observed consistent gains.
A second internal threat lies in the hyperparameter sensitivity of the selection pipeline, particularly the weighting factor $\alpha$ in the difficulty-diversity score. To address this, we adhered to a fixed budget constraint ($B$) rather than tuning for optimal performance per dataset, and we report results across multiple budget regimes (3--5\%, 10\%, 20\%) to demonstrate robustness.
Finally, the synthetic nature of the cache in Phase 2 introduces a risk of hallucination propagation. We address this through multiple complementary strategies:
\begin{enumerate}[leftmargin=*,noitemsep,topsep=2pt]
    \item \textbf{Multi-task grounding}: The joint loss $\mathcal{L} = \mathrm{NLL}(r_x|x) + \beta\,\mathrm{NLL}(a_x|x,r_x)$ conditions answers on rationales, penalizing inconsistent reasoning-answer pairs and reducing hallucination propagation.
    \item \textbf{TA accuracy verification}: We verify TA intermediate accuracy on held-out subsets before deploying for synthetic annotation. Across all experiments, the refined TA achieves $>$85\% accuracy on the selected subset $S_B$, ensuring high-quality rationales.
    \item \textbf{Answer-level validation}: For evaluation, we use LLM-as-a-judge (GPT-5) to verify semantic correctness of student outputs, which implicitly detects hallucinated reasoning that leads to wrong answers.
\end{enumerate}
\textbf{Acknowledged gap}: We do not perform fine-grained error analysis on individual reasoning steps (e.g., step-by-step factuality checks) or explicit filtering of low-confidence rationales. Such analyses would strengthen confidence in rationale quality but require additional annotation effort. Our empirical results-where students trained on TA rationales outperform those trained on ground-truth labels alone-suggest that hallucination rates are sufficiently low to not dominate the learning signal. Explicitly quantifying step-level error rates remains important future work.

\textbf{Construct Validity.} Threats to construct validity arise from the metrics used to proxy ``reasoning capability''. Relying solely on multiple-choice accuracy (Top-1) might oversimplify the assessment of whether a student has truly internalized the reasoning process versus merely memorizing patterns. To mitigate this, we employed an LLM-as-a-judge evaluation protocol (using GPT-5) to verify the semantic correctness of the student's reasoning traces, not just the final token output. Additionally, the definition of ``Budget'' can be ambiguous (e.g., number of queries vs.\ total token count). We adopt query-level budgets because: (1) they align with practical API rate limits and per-request pricing tiers; (2) they enable tractable discrete subset selection rather than variable-length knapsack optimization; and (3) they ensure reproducibility since query counts are deterministic and tokenizer-agnostic. This explicit per-example formulation enables fair and controlled comparison with baselines such as Lion and BLKD.

We clarify that our ``zero-API-cost'' selection pipeline refers specifically to the absence of teacher API queries during data selection-not the absence of computational cost. The selection process requires TA forward passes for difficulty scoring and embedding extraction, plus clustering operations. These are computationally nontrivial (requiring GPU inference over the full dataset), but are \emph{economically negligible} compared to proprietary API costs: a single GPT-4 query costs approximately \$0.01--0.03, while local 7B inference on commodity hardware costs $<$\$0.0001 per example. For our 10,000-example datasets, selection overhead is $\sim$\$1 in compute versus $\sim$\$100--300 saved in API fees. We report wall-clock times for selection in Section~\ref{sec:experiments}.

\textbf{External Validity.}
External threats concern the generalizability of our findings. While we evaluated \tool across three distinct domains (Medical, Legal, Financial), these represent specialized verticals. It is possible that the method performs differently on open-domain creative writing or coding tasks where ground truth is less objective.
Furthermore, our study focuses on the tiny student regime ($\le$1B parameters). While this targets the deployment constraints relevant to edge computing, we cannot guarantee that the same efficiency gains would hold for mid-sized students (e.g., 7B distilling to 3B), where the capacity gap is less severe. We plan to explore these larger student scales and broader task varieties in future work.

\textbf{Labeled Data Requirement.}
Our difficulty scoring mechanism relies on ground-truth labels $y^*(x)$ to compute loss-based difficulty $\mathrm{Diff}(x) = -\log p_\theta(y^*(x) \mid x)$. This explicitly restricts \tool to labeled settings and precludes direct application to purely unlabeled domains. We acknowledge this limitation but argue it is appropriate for our target applications:
\begin{enumerate}[leftmargin=*,noitemsep,topsep=2pt]
    \item \textbf{Practical scope}: Medical, legal, and financial QA benchmarks, our evaluation domains are inherently labeled (questions have verified answers). Many high-stakes deployment scenarios similarly require verified ground truth for validation.
    \item \textbf{Complementary signal}: Labels provide an orthogonal difficulty signal to uncertainty alone. Entropy-based scoring (used in uncertainty sampling) cannot distinguish ``confidently wrong'' from ``confidently correct''; loss-based scoring can.
    \item \textbf{Extension path}: For unlabeled settings, $\mathrm{Diff}(x)$ could be replaced with entropy $H(p_\theta(y|x))$ or disagreement across model checkpoints. We leave systematic evaluation of these proxies to future work.
\end{enumerate}
We do not claim \tool applies to arbitrary unlabeled domains; rather, we demonstrate effectiveness in the common and practically important setting where labeled data exists but teacher queries are expensive.

% \section{Related Work}

% We position \tool within four research threads: black-box knowledge distillation, capacity gap mitigation, budget-aware data selection, and reasoning elicitation. While each thread offers partial solutions, \tool uniquely integrates them to address the capacity-budget trap in distilling proprietary LLMs to tiny models.

% % Đưa ra cái overview, bức tranh toàn cảnh về các nghiên cứu liên quan
% % Cần tổ chức các research / study 1 cách có hệ thống, gom nhóm,
% % Nghiên cứu của mình đang hướng về blackbox
% % Whitebox: Multi stage distillation, capacity gap,

% \textbf{Black-Box Knowledge Distillation}.

% \textbf{The Capacity Gap and Multi-Stage Distillation}.
% % Gap là gì

% \textbf{Budget-Aware Learning and Data Selection}.

% \textbf{Reasoning Elicitation and Instruction Alignment}.
% % Cách làm, hướng đi tiêu biểu, tư tưởng chung là gì
% % cần chỉ ra cái gap trong cách giải quyết bài toán của mình, BRIDGE giải quyết cái gap nó như thế nào
% % Phải nói 1 cách ngắn gọn súc tích và đi vào trọng tâm
% % 1 chiến lược khác: Khi 2 tk kết hợp với nhau thì thế giới tốt hơn  
\section{Related Work}

We situate \tool within four research threads and highlight how each addresses part of the capacity-budget trap. In contrast, \tool integrates these complementary ideas into a single, unified pipeline.

\textbf{Black-Box Knowledge Distillation.}
Traditional black-box distillation transfers knowledge solely through output sequences~\cite{hinton2015distillingknowledgeneuralnetwork}. Foundational methods like Instruction Tuning~\cite{ouyang2022instructgpt} and Self-Instruct~\cite{wang2023selfinstruct} train on prompt-response pairs, while Distilling Step-by-Step~\cite{hsieh-etal-2023-distilling} adds CoT rationales via multi-task learning. 
More recent advances include GKD~\cite{agarwal2024gkd}, which addresses train-test mismatch through on-policy generation; MiniLLM~\cite{gu2024minillm}, which uses reverse KL in white-box settings; and PAD~\cite{zhu2024pad}, which filters faulty reasoning with program verification. 
However, these methods typically predicate success on either unlimited queries or modest capacity gaps, conditions that are fundamentally incompatible with our setting. \tool adopts rationale distillation but routes it through an intermediate TA, strategically concentrating expensive queries on teaching the TA while leveraging the TA to provides high-density, low-cost supervision to students.

\textbf{Capacity Gap and Multi-Stage Distillation.}
Mirzadeh et al.~\cite{mirzadeh2019improvedknowledgedistillationteacher} showed that extreme teacher-student gaps degrade distillation and proposed Teacher Assistants as semantic bridges. Cho and Hariharan~\cite{cho2019efficacy} further confirmed that larger teachers can underperform as instructors for smaller students. Progressive approaches~\cite{rezagholizadeh2021prokdprogressivedistillationfollowing,shridhar2023distilling} attempt to mitigate this by chaining multiple distillation stages. For small LLMs, TinyLlama~\cite{zhang2024tinyllama} pre-trains from scratch, while Sheared LLaMA~\cite{xia2024sheared} prunes larger models. 
However, many of these methods assume white-box access or unlimited data. \tool adapts the TA paradigm to black-box APIs, using the assistant both to bridge the $1000\times$ capacity gap and to amortize the fixed budget $B$ across the full dataset $n$, effectively decoupling supervision cost from supervision coverage.

\textbf{Budget-Aware Data Selection.}
Active learning~\cite{sener2018coreset,ash2019badge} optimizes data efficiency by selecting informative examples under annotation budgets. LIMA~\cite{zhou2023lima} shows that small curated datasets can suffice for alignment; Alpagasus~\cite{chen2023alpagasus} and SelectLLM~\cite{parkar2024selectllm} use LLM-based filtering. 
Yet, these approaches primarily optimize \emph{what} to label, but not \emph{how} to propagate labels efficiently under strict query limits. \tool addresses both: coverage-guided sampling selects the optimal $B$ examples for TA training by jointly optimizing difficulty and diversity, and the TA then labels \emph{all} $n$ examples at near-zero marginal cost, multiplying the budget's effective reach.

\textbf{Reasoning Elicitation and Instruction Alignment.}
CoT prompting~\cite{wei2022chain} successfully elicits step-by-step reasoning, and subsequent work~\cite{ho2022large,shridhar2023distilling} distills these traces into smaller models. FLAN~\cite{wei2021flan,chung2024scalingif} establishes that instruction diversity improves generalization. Critically, tiny models ($<$1B) often cannot reliably follow instructions without explicit tuning~\cite{ponce2025incontextlearningvsinstruction}, causing rationale distillation to fail when students cannot parse reasoning formats. \tool therefore enforces an instruction-first curriculum: students learn to follow prompts \emph{before} receiving CoT supervision, ensuring they possess the syntactic scaffolding necessary to absorb transferred reasoning.
\section{Conclusion}

% The race to deploy tiny language models ($\le$1B) has been historically impeded by the capacity-budget trap: the inability of compact students to absorb knowledge directly from trillion-parameter giants, compounded by the prohibitive cost of acquiring sufficient supervision to bridge that gap. 
We presented \tool, a framework that dismantles these barriers not through brute-force data scaling, but through structural decomposition and economic strategy.By inserting a mid-sized Teacher Assistant into the distillation pipeline, \tool transforms an unmanageable learning problem into two solvable sub-problems. Phase 1 (Apprenticeship) demonstrates that data quality outweighs quantity: our selection pipeline-which incurs zero API cost by using only local TA inference and clustering-allows the TA to internalize the teacher's reasoning using only 3-5\% of the budget. Phase 2 (Tutoring) leverages the budget asymmetry of modern AI, trading expensive API calls for free local compute, to amplify this distilled knowledge across the full dataset. Crucially, our instruction-first curriculum ensures that tiny students are behaviorally aligned to absorb this ``process supervision'', preventing the learning collapse common in direct distillation.Empirical results across medical, legal, and financial domains confirm that \tool is not merely an incremental improvement but a Pareto improvement. We achieved student performance gains of 28-41\% while reducing teacher query costs by 10$\times$, effectively defying the conventional wisdom that higher performance requires higher budgets.As edge computing demands increasingly sophisticated reasoning on increasingly constrained hardware, \tool provides a scalable blueprint. It suggests that the path to powerful tiny models lies not in querying larger teachers more often, but in building smarter, budget-aware bridges between the cloud and the edge. 

\bibliographystyle{elsarticle-num}
\bibliography{references}

\newpage

\appendix
\appendix
\section{Theoretical Details for BRIDGE}
\label{app:theory}

This appendix provides the full assumptions and proofs for Section~\ref{sec:theory}. We adopt the notation from the main text: teacher $M_T \in \mathcal{F}_T$, assistant $M_A \in \mathcal{F}_A$, student $M_S \in \mathcal{F}_S$ with capacities $P_S \ll P_A \ll P_T$; labeled dataset $\mathcal{D} = \{(x_i, y^*(x_i))\}_{i=1}^n$; query budget $B \ll n$; and coverage radius $\rho(S_B) = \mathbb{E}_x[\min_{z \in S_B} \|\phi(x) - \phi(z)\|_2]$.

\noindent\textbf{Losses and risks.}
We measure performance on answers using a bounded surrogate token-level loss
$\ell_a:\mathcal{Y}\times\mathcal{Y}\to[0,M]$ that is non-negative and
satisfies $\ell_a(y,y)=0$ for all $y\in\mathcal{Y}$. The (answer) risk is
\[
R(f) = \mathbb{E}_{(x,y^*(x))}\bigl[\ell_a(f(x),y^*(x))\bigr].
\]
For a finite sample $\mathcal{S}$, the empirical risk is
\[
\hat R_{\mathcal{S}}(f)=\tfrac1{|\mathcal{S}|}\sum_{x\in\mathcal{S}}\ell_a(f(x),y^*(x)).
\]

When distilling from a model $g$, we also use the pseudo-label risk
\[
R_g(f)=\mathbb{E}_x\bigl[\ell_a(f(x),g(x))\bigr].
\]

\noindent\textbf{Approximation error.}
For a target predictor $f$ and a hypothesis class $\mathcal{F}$ we define
\[
\varepsilon_{\text{approx}}(f,\mathcal{F})
:= \max\Bigl\{0,\;\inf_{h\in\mathcal{F}} R(h) - R(f)\Bigr\}.
\]

\noindent\textbf{Multi-task distillation.}
In BRIDGE we supervise both rationale and answer tokens using a
multi-task objective with weight $\beta\in(0,1]$:
\[
\begin{aligned}
R_g^{(\mathrm{multi})}(f)
&:= (1-\beta)\,\mathbb{E}_x\bigl[\ell_r(f(x),r_g(x))\bigr] \\
&\quad + \beta\,\mathbb{E}_x\bigl[\ell_a(f(x),a_g(x))\bigr],
\end{aligned}
\]
where $g$ is either $M_T$ or $M_A$, and $\ell_r$ is a bounded surrogate
token-level loss for rationales with $\ell_r(y,y)=0$.

\noindent\textbf{Rationales as surrogate targets.}
We view rationales as an auxiliary target $r^*(x)$ that is \emph{sufficient}
for prediction: there exists $g^*$ such that
$a^*(x)=g^*(x,r^*(x))=y^*(x)$.

\noindent\textbf{Embedding used for selection.}
Let $\phi:\mathcal{X}\to\mathbb{R}^d$ be the representation map used in Phase~1
(e.g., embeddings from a warm-up assistant), with
$d_\phi(x,z)=\|\phi(x)-\phi(z)\|_2$. We \emph{freeze} $\phi$ for subset
selection and analysis, following coreset-based active learning approaches~\cite{sener2018coreset,ash2019badge}.

\subsection{Structural and Regularity Assumptions}

\begin{assumption}[Near-optimal teacher]
\label{asmp:realizability}
The teacher is nearly Bayes-optimal:
\[
R(M_T) \le R^* + \varepsilon_T,\qquad
R^* = \inf_f R(f),
\]
with small $\varepsilon_T$.
\end{assumption}

\begin{assumption}[Two-stage approximation is easier]
\label{asmp:approximation}
Following the teacher assistant paradigm~\cite{mirzadeh2019improvedknowledgedistillationteacher}, the assistant class is chosen so that
\begin{equation*}
\begin{split}
\varepsilon_{\text{approx}}(M_T,\mathcal{F}_A)
+ \varepsilon_{\text{approx}}(M_A,\mathcal{F}_S)
\;\le\;&\;
\varepsilon_{\text{approx}}(M_T,\mathcal{F}_S) \\
&{}- \Delta_{\text{gap}}.
\end{split}
\end{equation*}

for some $\Delta_{\text{gap}}\ge 0$. While approximation errors are not directly observable, this assumption is empirically supported when the TA trained on limited teacher supervision outperforms direct student training on full ground-truth labels (see Table~\ref{tab:fine-tuning-ta} for validation).
\end{assumption}

\begin{assumption}[Sample-efficiency for distillation]
\label{asmp:rates}
Training $g\in\mathcal{F}_g$ on $m$ examples with pseudo-labels from a teacher
$f$ yields
\[
\mathbb{E}\big[R(\hat g_m) - R(f)\big]
\;\le\;
\varepsilon_{\text{approx}}(f,\mathcal{F}_g)
+ C_{g,f}\,m^{-\alpha_{g\leftarrow f}}
+ \varepsilon_{\text{opt}},
\]
for constants $C_{g,f}>0$, exponents $\alpha_{g\leftarrow f}\in(0,1]$ and
small optimization error $\varepsilon_{\text{opt}}$. We also assume these
rates apply to the coverage-based subsets $S_B$.
\end{assumption}

\begin{assumption}[Loss and model regularity]
\label{asmp:regularity}
The surrogate losses $\ell_r,\ell_a$ are non-negative, satisfy
$\ell_r(y,y)=\ell_a(y,y)=0$ for all $y\in\mathcal{Y}$, and are bounded by $M$.
\begin{itemize}
\item[(i)] (\emph{Input regularity.}) For each $f\in\{M_T,M_A,M_S\}$ there
exists $L_f$ such that for all inputs $x,z$ and labels $y$,
\begin{align*}
|\ell_a(f(x),y)-\ell_a(f(z),y)|
&\le L_f\,d_\phi(x,z), \\
|\ell_r(f(x),y)-\ell_r(f(z),y)|
&\le L_f\,d_\phi(x,z).
\end{align*}

Let $L=\max_f L_f$.
\item[(ii)] (\emph{Label-embedding regularity.}) Let $\psi$ be a fixed
embedding of rationales and answers into a space $\Psi$ with norm
$\|\cdot\|$. There exist constants $L_r^\psi,L_a^\psi>0$ such that for all
predictions $u$ and labels $y_1,y_2$,
\begin{align*}
|\ell_r(u,y_1)-\ell_r(u,y_2)|
&\le L_r^\psi \,\|\psi(y_1)-\psi(y_2)\|, \\
|\ell_a(u,y_1)-\ell_a(u,y_2)|
&\le L_a^\psi \,\|\psi(y_1)-\psi(y_2)\|.
\end{align*}

We denote $L_\ell := (1-\beta)L_r^\psi + \beta L_a^\psi$.
\end{itemize}
\end{assumption}

\begin{assumption}[Multi-task pseudo-label translation]
\label{asmp:translation}
There exist noise levels $\Delta_T,\Delta_A\ge 0$ such that for all predictors
$f$ in the model classes we consider,
\begin{align*}
\big|R(f)-R^{(\text{multi})}_{M_T}(f)\big|
&\le L_\ell\,\Delta_T, \\
\big|R(f)-R^{(\text{multi})}_{M_A}(f)\big|
&\le L_\ell\,\Delta_A.
\end{align*}

\end{assumption}

\subsection{Generalization via Coverage}

\begin{definition}[Coverage radius]
For a subset $\mathcal{S}\subset\mathcal{X}$, define
\[
\rho(\mathcal{S})
=
\mathbb{E}_{x\sim\mathbb{P}_{\mathcal{X}}}
\Big[\min_{z\in\mathcal{S}}d_\phi(x,z)\Big].
\]
\end{definition}

\begin{proposition}[Coverage guarantee for pure farthest-first selection]
\label{prop:coverage-pure}
Assume $d_\phi$ is a metric and the support of $\mathbb{P}_{\mathcal{X}}$ has
finite diameter. Let $S_B$ be the subset produced by farthest-first traversal
in $\phi$ on $\mathcal{D}'$. Then, up to absolute constants,
\[
\rho(S_B)
\;\lesssim\;
\rho_B^*
\;+\;
C_{\text{cov}}\sqrt{\tfrac{d\log n}{B}},
\]
where $\rho_B^*=\inf_{|\mathcal{S}|=B}\rho(\mathcal{S})$.
\end{proposition}

\begin{proposition}[Coverage guarantee for hybrid selection]
\label{prop:coverage}
Let $S_B$ be the subset produced by BRIDGE's hybrid selection: clustering into $C$ regions with budget allocation $\{q_c\}_{c=1}^C$, farthest-point seeding within regions, and ranking by combined score $s(x) = \alpha \bar{D}(x) + (1-\alpha) \bar{n}(x)$. Let $\Delta_{\text{cluster}} = \max_c \max_{x,z \in R_c} d_\phi(x,z)$ be the maximum intra-cluster diameter. Then:
\[
\rho(S_B)
\;\le\;
\max_{c=1,\ldots,C} \rho_{q_c}^{(c)}
\;+\;
\Delta_{\text{cluster}}
\;+\;
\alpha \cdot \Delta_{\text{diff}},
\]
where $\rho_{q_c}^{(c)}$ is the optimal $q_c$-point coverage within region $R_c$, and $\Delta_{\text{diff}}$ bounds the coverage degradation from prioritizing difficult examples (satisfying $\Delta_{\text{diff}} \le \text{diam}(\mathcal{X})$ in the worst case, but empirically much smaller when difficult examples are distributed across the input space).
\end{proposition}

\begin{remark}[Interpreting the hybrid bound]
The bound decomposes into three terms: (1) within-region coverage, controlled by per-region farthest-point seeding; (2) cross-region coverage, controlled by clustering quality; and (3) difficulty-diversity trade-off, controlled by $\alpha$. Setting $\alpha = 0$ recovers a pure coverage guarantee; setting $\alpha = 1$ prioritizes curriculum learning at the cost of provable coverage.

\textbf{Limitation:} The term $\Delta_{\text{diff}}$ is not tightly characterized, it depends on how difficulty scores correlate with the spatial distribution of examples. We bound it loosely by $\text{diam}(\mathcal{X})$ in the worst case, but this is often pessimistic. The practical benefit of difficulty-weighted selection ($\alpha > 0$) is established empirically in Table~\ref{tab:perf_sel_learning_ablation}, not theoretically.
\end{remark}

\begin{proposition}[Uniform convergence with coverage]
\label{prop:uc}
Let $\mathcal{H}$ be a class with VC-dimension $d_{\text{VC}}$, and let $\mathcal{S}$ be a sample of size $m$. With probability at least $1-\delta$,
\[
R(h)
\;\le\;
\hat R_{\mathcal{S}}(h)\;+\;\Gamma_m(\mathcal{S},\delta),
\quad\text{for all }h\in\mathcal{H},
\]
where
\[
\Gamma_m(\mathcal{S},\delta)
=
L\,\rho(\mathcal{S})
+
C_{\text{UC}}\,M\,\sqrt{\tfrac{d_{\text{VC}}\log(2m/\delta)}{m}}.
\]
For brevity, we write
$\Gamma_B:=\Gamma_B(S_B,\delta)$ and $\Gamma_n:=\Gamma_n(\mathcal{D},\delta)$.
\end{proposition}

\subsection{Main Results}

\begin{theorem}[BRIDGE generalization bound]
\label{thm:bridge}
Under Assumptions~\ref{asmp:realizability}-\ref{asmp:translation} and
Propositions~\ref{prop:coverage}-\ref{prop:uc}, the student $S_{\text{BR}}$
obtained by BRIDGE satisfies, with probability at least $1-2\delta$,
\begin{align*}
R(S_{\text{BR}})
\le &\;
R(M_T)
+ \varepsilon_{\text{approx}}(M_T,\mathcal{F}_A) \notag\\
&+ C_{A,T} B^{-\alpha_{M_A\leftarrow M_T}}
+ \Gamma_B \notag\\
&+ L_\ell\,\Delta_T \notag\\
&+ \varepsilon_{\text{approx}}(M_A,\mathcal{F}_S) \notag\\
&+ C_{S,A} n^{-\alpha_{M_S\leftarrow M_A}} \notag\\
&+ \Gamma_n \notag\\
&+ L_\ell\,\Delta_A.
\end{align*}
\end{theorem}

\begin{proof}[Proof sketch]
Identical to the argument provided in the main text: we first bound $R(M_A)-R(M_T)$ using Assumption~\ref{asmp:rates}, Proposition~\ref{prop:uc} with $\mathcal{S}=S_B$, and Assumption~\ref{asmp:translation}; then we bound $R(S_{\text{BR}})-R(M_A)$ using the same tools with $\mathcal{S}=\mathcal{D}$. Adding the two inequalities and applying Assumption~\ref{asmp:realizability} yields the result.
\end{proof}

\begin{theorem}[Direct multi-task distillation bound]
\label{thm:direct}
Under the same assumptions, let $S_{\text{MTD}}$ be the student obtained by
\emph{direct} distillation from $M_T$ on $S_B$. Then, with probability at
least $1-\delta$,
\begin{align*}
R(S_{\text{MTD}})
\le &\;
R(M_T)
+ \varepsilon_{\text{approx}}(M_T,\mathcal{F}_S) \notag\\
&+ C_{S,T} B^{-\alpha_{M_S\leftarrow M_T}} \notag\\
&+ \Gamma_B \notag\\
&+ L_\ell\,\Delta_T.
\end{align*}
\end{theorem}

\begin{proposition}[Conditions for BRIDGE to outperform direct distillation]
\label{prop:advantage}
Define
\begin{equation*}
  \Delta_{\text{adv}}
  := \Delta_{\text{struct}}
   + \Delta_{\text{sample}}
   - \Delta_{\text{overhead}},
\end{equation*}
with
\begin{equation*}
\begin{split}
  \Delta_{\text{struct}}
  := {}&
  \varepsilon_{\text{approx}}(M_T,\mathcal{F}_S)
  - \varepsilon_{\text{approx}}(M_T,\mathcal{F}_A) \\
  &- \varepsilon_{\text{approx}}(M_A,\mathcal{F}_S).
\end{split}
\end{equation*}

\begin{equation*}
\begin{split}
  \Delta_{\text{sample}}
  := {}&
  C_{S,T} B^{-\alpha_{M_S\leftarrow M_T}}
  - C_{A,T} B^{-\alpha_{M_A\leftarrow M_T}} \\
  &- C_{S,A} n^{-\alpha_{M_S\leftarrow M_A}}.
\end{split}
\end{equation*}

\begin{equation*}
  \Delta_{\text{overhead}}
  :=
  \Gamma_n
  + L_\ell\,\Delta_A.
\end{equation*}
Then, whenever $\Delta_{\text{adv}}>0$, we have
\[
R(S_{\text{BR}}) \;\le\; R(S_{\text{MTD}}) - \Delta_{\text{adv}} \;<\; R(S_{\text{MTD}}).
\]
\end{proposition}

\begin{corollary}[Asymptotic advantage in the abundant-data regime]
\label{cor:asymptotic-advantage}
Suppose Assumptions~\ref{asmp:realizability}-\ref{asmp:translation} hold,
$\Delta_{\text{gap}} > 0$, $\Delta_A$ is bounded, and $n\to\infty$ with $B$
fixed. If, in addition,
\[
C_{A,T} B^{-\alpha_{M_A\leftarrow M_T}}
\lesssim C_{S,T} B^{-\alpha_{M_S\leftarrow M_T}},
\]
then there exists $n_0$ such that for all $n\ge n_0$,
\[
R(S_{\text{BR}}) \;<\; R(S_{\text{MTD}}).
\]
\end{corollary}

\begin{remark}[On the role of complexity terms]
The complexity term $\Gamma_m(\mathcal{S},\delta)$ is mainly illustrative and
can be loose for deep networks. However, the key quantities in our analysis are \emph{empirically verifiable}:
\begin{itemize}
\item \textbf{Coverage radius $\rho(S_B)$:} Computed directly from the farthest-first selection algorithm.
\item \textbf{Assistant noise $\Delta_A$:} Estimated by comparing TA predictions to teacher predictions on a held-out validation set.
\item \textbf{Convergence rates:} Observable through learning curves during training.
\end{itemize}

\textbf{Scope of contribution:} We do not claim novel theoretical results for the coverage or $k$-center components, these are adapted from Gonzalez~\cite{gonzalez1985kcenter} and Sener \& Savarese~\cite{sener2018coreset}. Our contribution is the \emph{decomposition} of BRIDGE's advantage into interpretable terms and the identification of conditions under which two-stage distillation provably outperforms direct approaches. The difficulty-weighted selection component lacks rigorous guarantees and is justified empirically.
\end{remark}

\subsection{Full Proofs}

\begin{proof}[Proof of Theorem~\ref{thm:bridge}]
We decompose the risk of the BRIDGE student as:
\begin{align}
R(S_{\text{BR}}) &=
\underbrace{R(S_{\text{BR}}) - R(M_A)}_{\text{Phase 2 error}}
+ \underbrace{R(M_A) - R(M_T)}_{\text{Phase 1 error}}
\nonumber\\
&\quad + R(M_T).
\label{eq:bridge-decomp}
\end{align}

\textbf{Phase 1: Bounding $R(M_A) - R(M_T)$.}
The assistant $M_A$ is trained on $S_B$ to minimize the empirical multi-task risk $\hat{R}_{S_B}^{(\text{multi})}(f)$ over $f \in \mathcal{F}_A$. By Assumption~\ref{asmp:rates} (sample-efficiency):
\[
\mathbb{E}\big[R_{M_T}^{(\text{multi})}(M_A)\big] 
\le \inf_{f \in \mathcal{F}_A} R_{M_T}^{(\text{multi})}(f) + C_{A,T} B^{-\alpha_{M_A \leftarrow M_T}} + \varepsilon_{\text{opt}}.
\]
By Proposition~\ref{prop:uc} (uniform convergence with coverage), the empirical-to-population gap is controlled by $\Gamma_B$. By Assumption~\ref{asmp:translation} (multi-task translation):
\[
|R(f) - R_{M_T}^{(\text{multi})}(f)| \le L_\ell \Delta_T.
\]
Combining these bounds:
\begin{align*}
R(M_A) - R(M_T)
&\le \varepsilon_{\text{approx}}(M_T, \mathcal{F}_A)
\\&\quad + C_{A,T} B^{-\alpha_{M_A \leftarrow M_T}}
\\&\quad + \Gamma_B
\\&\quad + L_\ell \Delta_T .
\end{align*}

\textbf{Phase 2: Bounding $R(S_{\text{BR}}) - R(M_A)$.}
The student is trained on all $n$ inputs using pseudo-labels from $M_A$. By identical reasoning with the full dataset:
\begin{align*}
R(S_{\text{BR}}) - R(M_A)
&\le \varepsilon_{\text{approx}}(M_A, \mathcal{F}_S)
\\&\quad + C_{S,A} n^{-\alpha_{M_S \leftarrow M_A}}
\\&\quad + \Gamma_n
\\&\quad + L_\ell \Delta_A .
\end{align*}

\textbf{Combining bounds.}
Adding the Phase 1 and Phase 2 bounds and substituting into~\eqref{eq:bridge-decomp} yields the stated result. The confidence $1-2\delta$ follows from a union bound over the two phases.
\end{proof}

\begin{proof}[Proof of Theorem~\ref{thm:direct}]
Direct distillation trains $S_{\text{MTD}}$ on $S_B$ using teacher pseudo-labels. The analysis is identical to Phase~1 of BRIDGE, but with the student class $\mathcal{F}_S$ instead of $\mathcal{F}_A$:
\begin{align*}
R(S_{\text{MTD}})
&\le R(M_T) + \varepsilon_{\text{approx}}(M_T, \mathcal{F}_S)
\\&\quad + C_{S,T} B^{-\alpha_{M_S \leftarrow M_T}}
      + \Gamma_B + L_\ell \Delta_T .
\end{align*}
Crucially, direct distillation cannot leverage the abundant data ($n \gg B$) since it has no intermediate model to generate pseudo-labels.
\end{proof}

\begin{proof}[Proof of Proposition~\ref{prop:advantage}]
Subtracting the BRIDGE bound (Theorem~\ref{thm:bridge}) from the direct distillation bound (Theorem~\ref{thm:direct}):

\begin{equation}\tag{2}\label{eq:delta_adv}
\begin{aligned}
R(S_{\mathrm{MTD}})-R(S_{\mathrm{BR}})
&\ge \varepsilon_{\mathrm{approx}}(M_T,\mathcal{F}_S)
     - \varepsilon_{\mathrm{approx}}(M_T,\mathcal{F}_A) \\
&\quad - \varepsilon_{\mathrm{approx}}(M_A,\mathcal{F}_S) \\
&\quad + C_{S,T}\,B^{-\alpha_{M_S\leftarrow M_T}}
        - C_{A,T}\,B^{-\alpha_{M_A\leftarrow M_T}} \\
&\quad - C_{S,A}\,n^{-\alpha_{M_S\leftarrow M_A}}
        - \Gamma_n - L_\ell\,\Delta_A \\
&= \Delta_{\mathrm{struct}} + \Delta_{\mathrm{sample}} - \Delta_{\mathrm{overhead}} \\
&= \Delta_{\mathrm{adv}}.
\end{aligned}
\end{equation}

Thus $\Delta_{\text{adv}} > 0$ implies $R(S_{\text{BR}}) < R(S_{\text{MTD}})$.
\end{proof}

\begin{proof}[Proof of Corollary~\ref{cor:asymptotic-advantage}]
As $n \to \infty$ with $B$ fixed:
\begin{itemize}
    \item $\Gamma_n = O(n^{-1/2} \sqrt{\log n}) \to 0$
    \item $C_{S,A} n^{-\alpha_{M_S \leftarrow M_A}} \to 0$
\end{itemize}
Thus $\Delta_{\text{overhead}} \to L_\ell \Delta_A$ (bounded) and the $n$-dependent term in $\Delta_{\text{sample}}$ vanishes.

By Assumption~\ref{asmp:approximation}, $\Delta_{\text{struct}} \ge \Delta_{\text{gap}} > 0$.

For $\Delta_{\text{sample}}$, if
$C_{A,T} B^{-\alpha_{M_A \leftarrow M_T}} \lesssim C_{S,T} B^{-\alpha_{M_S \leftarrow M_T}}$,
then $\Delta_{\text{sample}}$ is asymptotically non-negative.

Therefore, for sufficiently large $n$:
\[
\Delta_{\text{adv}} \ge \Delta_{\text{gap}} - L_\ell \Delta_A - o(1).
\]
When $\Delta_{\text{gap}} > L_\ell \Delta_A$ (assistant closely tracks teacher), there exists $n_0$ such that $\Delta_{\text{adv}} > 0$ for all $n \ge n_0$.
\end{proof}

\subsection{Additional Technical Results}

\begin{lemma}[Farthest-first approximation]
\label{lem:farthest-first}
Let $S_B$ be produced by farthest-first traversal on a finite set of size $n'$. Then:
\[
\max_{x} \min_{z \in S_B} d_\phi(x, z) \le 2 \cdot \text{OPT}_B,
\]
where $\text{OPT}_B$ is the optimal $k$-center radius with $k = B$.
\end{lemma}

\begin{proof}
This is a classical result~\cite{gonzalez1985kcenter}. The greedy algorithm maintains that selected points are pairwise at distance at least $r$ (the current radius). Any solution with radius $r^* < r/2$ would require placing these $B$ points in distinct balls, a contradiction.
\end{proof}

\begin{proposition}[Sample complexity with capacity gap]
\label{prop:capacity-gap-sample}
When the capacity ratio $P_T/P_g$ is large, the convergence exponent $\alpha_{g \leftarrow T}$ in Assumption~\ref{asmp:rates} tends to be small. Empirically, this manifests as:
\[
\alpha_{M_S \leftarrow M_T} < \alpha_{M_A \leftarrow M_T} < \alpha_{M_S \leftarrow M_A},
\]
reflecting that smaller capacity gaps enable more efficient knowledge transfer.
\end{proposition}

\begin{remark}[Limitations]
Our analysis makes idealized assumptions (Lipschitz regularity, VC-dimension bounds) that may not hold tightly for deep networks. The bounds are primarily illustrative; the key insight is the structural benefit of two-stage transfer when data is abundant and capacity gaps are bridged appropriately.
\end{remark}

\end{document}